\documentclass{article} % For LaTeX2e
\usepackage[utf8]{inputenc} % allow utf-8 input
\usepackage[T1]{fontenc}    % use 8-bit T1 fonts
\usepackage{hyperref}       % hyperlinks
\usepackage{url}            % simple URL typesetting
\usepackage{booktabs}       % professional-quality tables
\usepackage{multirow}
\usepackage{amsfonts}       % blackboard math symbols
\usepackage{nicefrac}       % compact symbols for 1/2, etc.
\usepackage{microtype}      % microtypography
\usepackage{color, colortbl} % colored cells in tables
\usepackage{wrapfig,lipsum}
\usepackage{amsthm,mathrsfs,amsmath}
\usepackage{amssymb}
\usepackage{paralist}
\usepackage{tcolorbox}
\usepackage{pifont}
\usepackage{sectsty}
\usepackage{xargs}
\usepackage{dsfont}
\usepackage{caption}
\usepackage{subcaption} % graphicx,
\usepackage{upgreek}
\usepackage{mathrsfs}
\usepackage{enumitem}
\usepackage{aliascnt}
\usepackage[noend]{algpseudocode}  % Eviter les endif, endfor
\usepackage{algorithm}
\usepackage{bbm}
\usepackage{bm}
\usepackage{iclr2025_conference,times}

\definecolor{mydarkblue}{rgb}{0,0.08,0.45}
\definecolor{dmorange500}{HTML}{FF5F19}
\definecolor{dmblue300}{HTML}{2267EB}
\definecolor{dmred300}{HTML}{FF617B}
\hypersetup{colorlinks, citecolor=dmblue300, urlcolor=mydarkblue}
\graphicspath{{figures/}}

\usepackage[capitalize,noabbrev]{cleveref}

\usepackage[textsize=tiny]{todonotes}
\usepackage{xcolor} % define colors in text

\iclrfinalcopy
\theoremstyle{plain}
\newtheorem{theorem}{Theorem}[section]

\newtheorem{lemma}[theorem]{Lemma}
\newtheorem{corollary}[theorem]{Corollary}
\theoremstyle{definition}

\newtheorem{assumption}{\textbf{H}\hspace{-3pt}}
\crefformat{assumption}{{\textbf{H}}#2#1#3}
\theoremstyle{remark}

\newtheorem{example}[theorem]{Example}

\definecolor{aurometalsaurus}{rgb}{0.43, 0.5, 0.5}
\definecolor{britishracinggreen}{rgb}{0.0, 0.26, 0.15}
\definecolor{burntumber}{rgb}{0.54, 0.2, 0.14}
\definecolor{cobalt}{rgb}{0.0, 0.28, 0.67}
\definecolor{bulgarianrose}{rgb}{0.28, 0.02, 0.03}
\definecolor{ceruleanblue}{rgb}{0.16, 0.32, 0.75}
\definecolor{darkgreen}{RGB}{0,128,0}

\newcommand{\N}{\mathbb{N}}

\newcommand{\R}{\mathbb{R}}

\newcommand{\E}{\mathbb{E}}
\newcommand{\XC}{\mathbb{R}^{d}}
\newcommand{\YC}{\mathcal{Y}}
\newcommand{\ZC}{\mathcal{Z}}

\newcommand{\argmin}{\operatornamewithlimits{\arg\min}}

\newcommand{\Oh}{\operatorname{\mathrm{O}}}

\newcommand{\prob}{\mathbb{P}}

\newcommand{\rmd}{\mathrm{d}}
\newcommand{\1}{\mathds{1}}

\newcommand{\pr}[1]{\left({#1}\right)}
\newcommand{\prt}[1]{({\textstyle{#1}})}
\newcommand{\prn}[1]{({#1})}
\newcommand{\prbig}[1]{\big({#1}\big)}
\newcommand{\prBig}[1]{\Big({#1}\Big)}

\newcommand{\br}[1]{\left[{#1}\right]}

\newcommand{\brbig}[1]{\big[{#1}\big]}

\newcommand{\ac}[1]{\left\{{#1}\right\}}
\newcommand{\act}[1]{\{{\textstyle{#1}}\}}

\newcommand{\abs}[1]{\left\lvert{#1}\right\rvert}

\newcommand{\gauss}{\mathcal{N}}

\newcommand{\q}[1]{Q_{#1}}
\renewcommand{\algorithmiccomment}[1]{\textbf{\textcolor{darkgreen}{// \texttt{#1}}}}
\newcommand{\tv}{\mathrm{d}_{\mathrm{TV}}}
\def\iid{i.i.d.}

\newcommand{\confReg}{\mathcal{R}}
\newcommand{\emptysymbol}{\emptyset}
\newcommand{\pred}{\mathrm{pred}}
\newcommand{\algo}{\texttt{CP}$^2$}
\newcommand{\hpdalgo}{\texttt{CP}$^2$\texttt{-HPD}}
\newcommand{\pcpalgo}{\texttt{CP}$^2$\texttt{-PCP}}
\newcommand{\PZ}{\Pi_{Z\mid X}}
\newcommand{\distr}{\Pi}
\newcommand{\gammayx}{\Pi_{Y\mid X}} 
\newcommand{\pdfdistr}{\gamma} 
\newcommand{\adjfunc}[1]{f_{#1}}
\newcommand{\tcount}{n}
\newcommand{\thmspace}{\vspace{0.4em}}
\newcommand{\paragraphformat}[1]{\paragraph{#1}}  % {\vspace{.2em}\textbf{#1}\hspace{0.1em}} {\paragraph{#1}}

\newcommand{\measure}{\mu}

\begin{document}
% !TEX root = main.tex

\title{\center Probabilistic Conformal Prediction with \\ Approximate Conditional Validity}

\author{
  \hspace*{-1.2em}
  Vincent Plassier$^{1}$ % \footnote{\texttt{vincent.plassier@ens-paris-saclay.fr}}
  \,
  % \And
  Alexander Fishkov$^{2,3}$
  \,  
  % \And
  Mohsen Guizani$^{3}$
  \,  
  % \AND
  Maxim Panov$^{3}$
  \,  
  % \And
  Eric Moulines$^{3,4}$
  \vspace*{.3em}
  \\
  \hspace*{-2.3em}
  {}$^{1}$ Lagrange Mathematics and Computing Research Center 
  {}$^{2}$ Skolkovo Institute of Science and Technology
  \vspace*{.3em}
  \\
  \hspace*{.7em}
  {}$^{3}$ Mohamed bin Zayed University of Artificial Intelligence
  {}$^{4}$ CMAP, Ecole Polytechnique
}

\maketitle

\begin{abstract}
  We develop a new method for generating prediction sets that combines the flexibility of conformal methods with an estimate of the conditional distribution $\textup{P}_{Y \mid X}$. Existing methods, such as conformalized quantile regression and probabilistic conformal prediction, usually provide only a marginal coverage guarantee. In contrast, our approach extends these frameworks to achieve approximately conditional coverage, which is crucial for many practical applications. Our prediction sets adapt to the behavior of the predictive distribution, making them effective even under high heteroscedasticity.
  While exact conditional guarantees are infeasible without assumptions on the underlying data distribution, we derive non-asymptotic bounds that depend on the total variation distance of the conditional distribution and its estimate. Using extensive simulations, we show that our method consistently outperforms existing approaches in terms of conditional coverage, leading to more reliable statistical inference in a variety of applications.
 \end{abstract}

\section{Introduction}
\label{sec:introduction}
% !TEX root = ../main.tex

Conformal predictions are commonly used to construct prediction sets. Under minimal assumptions, they offer finite-sample validity~\citep{vovk2005algorithmic,shafer2008tutorial}. However, significant challenges arise with high heteroskedasticity, often leading to incorrect inferences~\citep{dewolf2023heteroskedastic}. The split-conformal approach uses a set of $\tcount$ calibration data points $\{(X_k,Y_k)\}_{k\in[\tcount]}$ with $X_k \in \XC$ and $Y_k \in \YC$ to create a prediction set $\mathcal{C}_{\alpha}(x)$ where $\alpha \in (0,1)$. For each $x\in\XC$, the prediction set based on a conformity score function $V\colon \XC\times\YC\to \R$, is given by
\begin{equation}
  \mathcal{C}_{\alpha}(x)
  = \ac{y\in\YC\colon V(x,y)\le \q{1-\alpha}\pr{\frac{1}{\tcount+1} \sum\nolimits_{k=1}^{\tcount} \delta_{V(X_k,Y_k)} + \frac{1}{\tcount+1} \delta_{\infty}}},
\label{eq:split_pred_set}
\end{equation}
where \(\delta_x\) is the Dirac mass and $\q{1-\alpha}$ is the $(1 - \alpha)$-quantile of the adjusted empirical score distribution $\frac{1}{\tcount+1} \sum_{k=1}^{\tcount} \delta_{V(X_k,Y_k)} + \frac{1}{\tcount+1} \delta_{\infty}$.
If the calibration data $\{(X_k,Y_k)\}_{k\in[\tcount]}$ is drawn \iid\ from a population distribution $\textup{P}_{X,Y}$, then for any new data point $(X_{\tcount+1},Y_{\tcount+1})\sim \textup{P}_{X,Y}$ sampled independently of the calibration data, the conformal theory ensures the \textit{marginal validity} of $\mathcal{C}_{\alpha}(X_{\tcount+1})$, meaning that
$\prob\pr{Y_{\tcount+1}\in \mathcal{C}_{\alpha}(X_{\tcount+1})} \ge 1 - \alpha$.
%
%Initially, most of the conformal methods focused on estimating a mean regression function for $Y \mid X$, to then construct a fixed-width band around it; see~\citet{vovk1999machine,vovk2005algorithmic}. 
This marginal guarantee can hide significant discrepancies in the coverage of different regions of the input space $\XC$; see e.g.~\citet{izbicki2022cd}. %Conditional conformal prediction methods allow to construct a confidence interval that adapt to the test point under consideration $X_{\tcount+1}$.
%For example, if a patient has certain characteristics, the goal is to guarantee that their treatment will be successful with a confidence of $1-\alpha$. Therefore, we want to adapt the guarantees to the profile of the individual instead of giving guarantees that only apply to the population as a whole.
Conditional validity is a more desirable guarantee than marginal validity:
for any $x \in \XC$, the set $\mathcal{C}_{\alpha}(x)$ is  \textit{conditionally valid} if 
\begin{equation}
\label{eq:conditional-validity}
\prob\pr{Y_{\tcount+1}\in \mathcal{C}_{\alpha}(X_{\tcount+1}) \mid X_{\tcount+1}=x}
  \ge 1 - \alpha.
\end{equation}
However, this property cannot be achieved without further assumptions about the data distribution; see~\citet{vovk2012conditional,lei2014distribution}.
For practical purposes, it is enough to construct sets $\mathcal{C}_{\alpha}$ which approximate~\eqref{eq:conditional-validity}, and ideally achieve it asymptotically under suitable conditions in the limit of large sample size $n$.

\textbf{Related work.} A great deal of research has been devoted to this problem, starting with the case where $\YC=\R$.
For example \citet{romano2019conformalized,kivaranovic2020adaptive} proposed methods based on estimate of the lower and upper conditional quantile function $\hat{q}_{\alpha/2}$ and $\hat{q}_{1-\alpha/2}$, to define a quantile-based conformity score:
$V(x,y) = \max\ac{\hat{q}_{\alpha/2}(x)-y, y-\hat{q}_{1-\alpha/2}(x)}$ which is then conformalized. 
Improvements of this conformity score are  investigated in~\citet{kivaranovic2020adaptive,sesia2020comparison}. \citet{sesia2020comparison} have shown that the constructed interval converges to the narrowest possible interval that achieve conditional coverage under mild assumptions. We stress that these methods are specific to the case $\YC=\R$; in addition, when the conditional distribution $\textup{P}_{Y \mid X}$ is multimodal, restricting  prediction to intervals is suboptimal; see~\citet{wang2023probabilistic} for a discussion and examples.
%\begin{wrapfigure}{r}{0.25\textwidth}
%    \centering
%    % \vskip -.5em
%    \includegraphics[width=.25\textwidth,height=.23\textwidth]{bimodal-continous.pdf}
%    % \vskip -1em
%    \caption{Bimodal example.}
%    \label{fig:bimodal}
%\end{wrapfigure}
%\vspace{-10pt}
It has been suggested in~\citet{guan2023localized,alaa2023conformalized} to partition $\XC$ and learn a specific quantile for each  regions. Splitting the space $\XC$ into multiple regions typically leads to an increase in the length of the prediction set; see~\citet{romano2020malice,melki2023group} for comments.

A number of studies have focused on the construction of prediction sets using an estimator of the conditional distribution $\textup{P}_{Y \mid X}$. Again, in the case where $\YC=\R$, \citet{cai2014adaptive,lei2014distribution} have constructed prediction intervals based on an estimator of  the conditional density and have established the asymptotic validity under appropriate conditions. \citet{han2022split} use kernel density estimation to construct asymmetric prediction bands. However, this method cannot handle bimodality as it generates a single interval.
On the other hand, \citet{sesia2021conformal} partition the domain of $Y$ into bins to create a histogram approximation of $\textup{P}_{Y \mid X}$. The authors showed that their method satisfies the marginal validity while achieving the asymptotic conditional coverage; see also~\citet{lei2018distribution}. Asymptotic conditional coverage is also obtained in~\citet{sesia2020comparison,cauchois2020knowing} using quantile regression-based methods, or using cumulative distribution function estimators~\citep{izbicki2019flexible,chernozhukov2021distributional}. Conditionally valid prediction sets have been shown to improve the robustness to perturbations~\citep{gendler2021adversarially}.
\citet{guha2024conformal} proposed a novel approach that converts regression tasks into classification problems by binning the output space and discretizing the labels. Leveraging this discretization, they approximate the conditional density to construct prediction sets that correspond to regions of Highest Predictive Density (HPD). A key limitation of these methods lies in the discretization process, as the number of labels required for complex scenarios can become computationally prohibitive.  \citet{diamant2024conformalized} introduced an approach that estimates conditional densities using neural networks parameterized by splines, offering a more flexible representation. 

Very few studies have addressed the scenario where the prediction target is multi-dimensional, i.e., \(\mathcal{Y} = \R^q\) with \(q > 1\). 
\citet{wang2023probabilistic} developed the \texttt{PCP} method, based on implicit conditional generative models (CGMs). These CGMs allow for the generation of samples from the conditional distribution without requiring an explicit closed-form expression. The \texttt{PCP} method constructs prediction sets as unions of balls, whose centers are generated from the CGM. However, in \texttt{PCP}, the radius of these balls is fixed across the space, which can introduce significant limitations. In regions with low variability, a fixed radius may result in over-coverage, while for highly dispersed conditional distributions, it may lead to under-coverage, failing to capture the full extent of the relevant space.
We observed that the performance of \texttt{PCP} deteriorates as the number of balls increases, exacerbating the heteroscedasticity problem. This underscores the need for a more adaptive methodology that can dynamically adjust the size of prediction sets in response to local variability in the data. 

Our work addresses these challenges through the following main \textbf{contributions}:
\begin{itemize}
  \item We propose a new \algo\ method for constructing conditional confidence sets that adapts to the local structure of the data distribution, capable of addressing both classical regression problems where $\YC = \R$ and more complex multi-dimensional prediction tasks where $\YC = \R^q$. Our approach is versatile in accommodating scenarios involving either an explicit conditional density estimator or an implicit generative model; see Section~\ref{sec:ccp}.

  \item We develop a theoretical framework to analyze the properties of the proposed \algo\ method, establishing both its marginal and approximate conditional validity. Furthermore, we demonstrate that asymptotic conditional coverage is attainable under a weak consistency assumption on the predictive distribution; see Section~\ref{sec:theory}.

  \item We demonstrate the effectiveness of the proposed method through a series of experiments on synthetic and real-world datasets. The results indicate that our approach consistently outperforms existing methods in terms of conditional coverage. Specifically, it excels in handling classical regression problems, effectively addressing multimodality, and proves robust in the more challenging setting of multidimensional prediction tasks; see Section~\ref{sec:expriments}.
\end{itemize}

\section{The \algo\ framework}
\label{sec:ccp}
% !TEX root = ../main.tex

We want to construct marginally valid predictive sets with approximate conditional validity. We follow the split-conformal approach to conformal inference due to its computational feasibility with large datasets; see along others~\citet{papadopoulos2002inductive,papadopoulos2008inductive,papadopoulos2011regression,romano2019conformalized,kivaranovic2020adaptive}. The first step  involves splitting the data samples into two disjoint subsets, the training set $\mathcal{T} =\{(\tilde{X}_k,\tilde{Y}_k)\}_{k=1}^m$ and calibration set $\mathcal{C} = \{(X_k,Y_k) \}_{k=1}^n$. It is assumed in the sequel that the training and calibration data are mutually independent and \iid\ with distribution $\textup{P}_{X,Y}$ over the feature vectors $X \in \XC$ and response variables $Y\in\YC$.
The target set $\YC$ can be either finite or continuous. We must now predict the unknown value of $Y_{\tcount+1}$ given an input $X_{\tcount+1}$, independent from $\mathcal{T} \cup \mathcal{C}$. Our goal is to construct a prediction set $\mathcal{C}_{\alpha}(X_{\tcount+1})$ that contains the unobserved output $Y_{\tcount+1}$ with probability close to $1-\alpha$, where $\alpha \in (0,1)$ is the user-specified confidence level. 

An estimator $\distr_{Y \mid X}$ of the conditional probability $\textup{P}_{Y \mid X}$ is learnt using the training data. There is a rich body of research on nonparametric conditional density estimation, with the most common methods relying on smoothing techniques such as kernel smoothing and local polynomial fitting.
An alternative approach involves transforming the conditional density estimation task into a regression problem, allowing the application of nonparametric regression methods to approximate the conditional density. More recently, generative methods leveraging deep neural networks have been developed for nonparametric conditional density estimation, which enable sampling from the conditional distribution; see~\citet{abadi2016deep, zhou2021deepgenerative} for examples of these techniques. 
In the sequel, the choice of this algorithm is treated as a black box. There are three main ingredients for our approach:
\begin{enumerate}[leftmargin=1.5em] % , itemsep=0pt, parsep=0pt, partopsep=0pt
  \item In classical conformal prediction methods, the shape of the prediction set is specified by the score function $V(x, y)$; see~\eqref{eq:split_pred_set}. The first element of our construction is a family of explicitly given confidence sets $\confReg_{z}(x; t)$ parameterized by $t \in \mathsf{T}$ where $\mathsf{T}$ is a subset of $\R$ and $z \in \mathcal{Z}$ auxiliary variables. The index set $\mathsf{T}$ can, in most cases, be taken as either $\mathsf{T} = \R$ or $\mathsf{T} = \R_+$. The key assumptions for the confidence set are as follows:
  (a) The size of $\confReg_{z}(x; t)$ increases with $t \in \mathsf{T}$ for any $z \in \ZC$. In addition, by choosing a sufficiently large value of $t$, the entire output space $\YC$ can be covered.
  (b) There exists a form of continuity for $t\mapsto \confReg_{z}(x;t)$. In mathematical terms, the following assumption should hold:

\begin{assumption}\label{ass:confReg}
  For any $(x,z)\in\XC\times \ZC$, the confidence sets $\{ \confReg_{z}(x; t) \}_{t \in \mathsf{T}}$ are non-decreasing, $\distr_{Y|X=x}(\cap_{t\in\mathsf{T}} \confReg_{z}(x; t))=0$, $\cup_{t\in\mathcal{T}} \confReg_{z}(x; t)=\YC$; in addition, for any $t \in \mathsf{T}$, $\cap_{t'>t}\confReg_{z}(x;t')=\confReg_{z}(x; t)$.
\end{assumption}

  \begin{example}
    For instance, $\confReg(x;t)$ can be chosen as a ball centered around an estimate of conditional mean $\textup{P}_{Y \mid X}$ with radius $t$ and there is no auxiliary variables then (we then remove subscript $z$). Examples of confidence intervals specialized to the case where $\mathcal{Y}=\R$ are given in \Cref{table:confReg-examples}.
  \end{example}

  \begin{example}
    If the predictive distribution is multimodal, a ball  centered around the predictive mean often fails to provide an informative prediction set. Ideally, $\confReg_z(x; t)$ should correspond to the set with the highest predictive density (HPD) of $\textup{P}_{Y \mid X}$. However, HPD regions are difficult to determine in practice, even when the conditional predictive density is available. Following \cite{wang2023probabilistic}, we may define the prediction sets  as $\confReg_{z}(x; t):=\cup_{i=1}^{M}\mathrm{B}(y_i, t)$, and they depends on an exogenous variables $z=(y_1,\ldots,y_{M}) \in \ZC$, where each $y_i$ is sampled conditionally independently from the conditional generative model $\distr_{Y \mid X=x}$. 
  \end{example}
  In a general setting, the auxiliary variables $z \in \mathcal{Z}$ are sampled according to a kernel $\bar{\distr}_{Z \mid X=x}$ for a given $x$. 

  \item To localize conformal prediction methods, it is convenient to introduce a function $\adjfunc{\tau}(\lambda)$ parameterized by $\tau$. Such function aims to transform the conformity score $\lambda$ and was introduced (albeit in a slightly different form) in~\citet{deutschmann2023adaptive,han2022split}. Examples of such a function are $\adjfunc{\tau}(\lambda) = \tau \lambda$ and $f_\tau(\lambda)= \tau + \lambda$. We assume that:
  \begin{assumption}\label{ass:tau}
    There exists $\varphi\in\mathsf{T}$ such that $\tau \in\mathsf{T} \mapsto \adjfunc{\tau}(\varphi)$ is increasing and bijective.
    In addition, $\lambda \in\mathsf{T} \mapsto \adjfunc{\tau}(\lambda)$ is increasing for any $\tau\in\mathsf{T}$.
  \end{assumption}
  We define $\tau_{x,z}$ using the estimated predictive density $\distr_{Y \mid X=x}$ according to 
  \begin{equation}\label{eq:def:tau}
    \tau_{x,z} = \inf\ac{\tau \in \mathsf{T} \colon \distr_{Y \mid X=x}(\confReg_{z}(x; \adjfunc{\tau}(\varphi)))\ge 1-\alpha}.
  \end{equation} 
  It is easily shown that for any $\alpha\in(0,1)$, $x\in\XC$, $z\in\ZC$, then $\tau_{x,z} \in \mathsf{T}$ and $\distr_{Y \mid X=x}(\confReg_{z}(x;\adjfunc{\tau_{x,z}}(\varphi)))\ge 1-\alpha$; see~\Cref{lem:tau}.

  \item Finally, we need to introduce the conformity score for the considered setup. The natural choice is the minimal size of the set required to cover the observation $y$ at the input $x$ for the auxiliary variables $z$:  $\lambda_{x,y,z} = \inf\ac{t \in \mathsf{T} \colon y \in \confReg_z(x; t)}$.

  \item The resulting procedure works as follows. For $k \in \{1,\dots,n\}$, we set $\bar{\tau}_{k} := \tau_{X_k,Z_k}$ and $\bar{\lambda}_{k} := \lambda_{X_k,Y_k,Z_k}$ where $\{Z_k\}_{k=1}^n$ are sampled conditionally independently from $\bar{\distr}_{Z \mid X= X_k}$. Given $X_{\tcount+1}\in\XC$, we sample  $Z_{\tcount+1}\sim \bar{\distr}_{Z\mid X=X_{\tcount+1}}$ conditionally independently from $\{(X_k,Y_k,Z_k)\}_{k=1}^{n}$,   
and construct the resulting \algo\ prediction set as
\begin{equation}\label{eq:pred-set}
  \textstyle
  \mathcal{C}_{\alpha}(X_{\tcount+1})
  = \confReg_{Z_{\tcount+1}}\pr{X_{\tcount+1}; \adjfunc{\bar{\tau}_{\tcount+1}}\bigl(\q{1-\alpha}(\measure_{\tcount})\bigr)},
\end{equation}
where $\q{1-\alpha}(\measure_{\tcount})$ is the $1-\alpha$ quantile of the distribution $\measure_{\tcount}$, and is given by
\begin{equation}\label{eq:def:mu}
  \measure_{\tcount} = \frac{1}{\tcount+1}\sum\nolimits_{k=1}^{\tcount} \delta_{\adjfunc{\bar{\tau}_k}^{-1}(\bar{\lambda}_k)} + \frac{1}{\tcount+1} \delta_{\infty}.
\end{equation}

 The transformation $\{v\mapsto \adjfunc{\tau}(v)\}_{\tau\in\mathsf{T}}$ balances the following two factors: (a) The optimal parameter $\lambda_{x,y,z}$ ensuring that $y$ is included in the confidence set $\confReg_{z}(x; \lambda_{x,y,z})$; (b) The parameter $\tau_{x,z}$ obtained from the probabilistic model $\Pi_{Y\mid X=x}$.
\end{enumerate}

\begin{table}[t]
  \centering
  \caption{Confidence sets $\confReg(x;t)$ found in the literature and also discussed in~\cite{gupta2022nested}.}  % \cite[Table~1]{gupta2022nested}
  \makebox[\linewidth]{
  \resizebox{1.0\textwidth}{!}{
  {
  \begin{tabular}{ccc}
    \toprule
    \cite{lei2018distribution} & \cite{lei2018distribution} & \cite{kivaranovic2020adaptive} \\
    $[\pred(x)-t, \pred(x)+t]$
    & $[\pred(x)-t \sigma(x), \pred(x)+t \sigma(x)]$ 
    & $(1+t)[q_{\alpha / 2}(x), q_{1-\alpha / 2}(x)]-t q_{1 / 2}(x)$ 
    \\
    \midrule
    \cite{chernozhukov2021distributional} & \cite{romano2019conformalized} & \cite{sesia2020comparison} \\
    $[q_{t}(x), q_{1-t}(x)]$
    & $[q_{\alpha / 2}(x)-t, q_{1-\alpha / 2}(x)+t]$ 
    & $[q_{\alpha / 2}(x), q_{1-\alpha / 2}(x)] \pm t (q_{1-\alpha / 2}(x)-q_{\alpha / 2}(x))$
    \\
    \bottomrule
  \end{tabular}}
  }
  }
  \label{table:confReg-examples}
\end{table}
We stress that \algo\ is a general framework that can be adapted to many choices for conditional predictive density estimates, constructing the family of confidence sets, and selecting the calibration function $f_\tau(\lambda)$. However, we start with the simple example that shows that \algo\ is more general than the classical split-conformal CP approach.

\begin{wrapfigure}{r}{0.46\textwidth}
  \vskip-1.5em
  \centering
  \includegraphics[width=.47\textwidth]{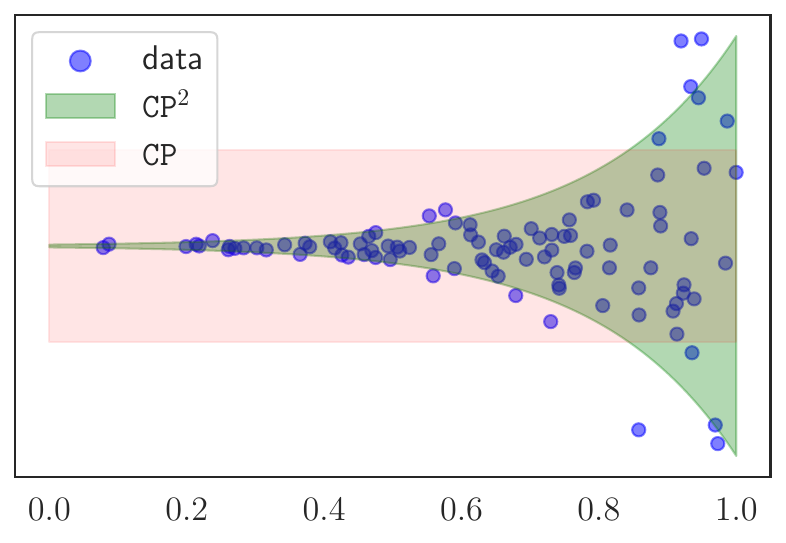}
  \caption{Predictions sets obtained via the standard \texttt{CP} and \algo\ methods.}
  \label{fig:funnel}
  \vskip-1em
\end{wrapfigure}

\paragraphformat{Simple example of \algo.}
We begin with a simple application of \algo\ to highlight its differences from the basic conformal approach, with $\YC=\R$. The calibration sets are defined as: $\mathcal{R}(x; t) = \{y \in \YC \colon |y - \text{pred}(x)| \leq t\}$ for $t \in \R_+$. There are no auxiliary variables $z$ in this case, so we omit $z$ from the notation. Assumption~\Cref{ass:confReg} is easily satisfied with $\mathsf{T} = \R_+$. We then take $f_\tau(\lambda) = \tau \lambda$, $\tau \in \R_+$  and $\varphi = 1$: \Cref{ass:tau} is also satisfied. Note that $f^{-1}_\tau(\lambda)= \lambda/\tau$ for $\tau \in \R_+^*$.
Using the \algo\ approach, we find that $\lambda_{x,y} = |y - \pred(x)|$, which corresponds to a standard conformity score. The classical conformal prediction method defines the $1-\alpha$ quantile based on the associated empirical measure $\nu_n = \frac{1}{n+1} \sum_{k=1}^n \delta_{\bar{\lambda}_k} + \frac{1}{n+1} \delta_\infty$, where $\bar{\lambda}_k = |Y_k - \pred(X_k)|$. \algo\ differs from the basic conformal approach by introducing $\tau_{x} = \argmin\{\tau \in \mathsf{T} \colon \distr_{Y \mid X=x}([\pred(x) \pm \tau] ) \geq 1 - \alpha\}$ as in~\eqref{eq:def:tau}, where $[\pred(x) \pm \tau]= [\pred(x)-\tau,\pred(x)+\tau]$.
The prediction set becomes $[\pred(x) \pm \adjfunc{\tau_{x}}(\q{1-\alpha}(\measure_{\tcount}))]$, where $\measure_{\tcount}= \frac{1}{n+1} \sum_{k=1}^n \delta_{\bar{\lambda}_k/\tau_k} + \frac{1}{n+1} \delta_\infty$.
%%%%%%%%%%%%%%%%%%% REMETTRE %%%%%%%%%%%%%%%%%%%
In this setting, the choice of $\varphi$ is irrelevant. Take $\varphi >0$ and denote $\tau^\varphi_{x} = \argmin\{\tau \in \mathsf{T} \colon \distr_{Y \mid X=x}([\pred(x) \pm \adjfunc{\tau}(\varphi)] ) \geq 1 - \alpha\}$. It is easily seen that $\tau_x^{\varphi}= \tau_x / \varphi$. The prediction set becomes  
$[\pred(x) \pm \adjfunc{\tau^\varphi_{x}}(\q{1-\alpha}(\measure^\varphi_{\tcount}))]$, where $\measure^\varphi_{\tcount}= \frac{1}{n+1} \sum_{k=1}^n \delta_{\bar{\lambda}_k/\bar{\tau}_k^\varphi} + \frac{1}{n+1} \delta_\infty$ with $\bar{\tau}_k^\varphi= \tau^\varphi_{X_k}$. Note that $\q{1-\alpha}(\measure^\varphi_{\tcount}))= \varphi  \q{1-\alpha}(\measure_{\tcount}))$ and thus
$\adjfunc{\tau^\varphi_{x}}(\q{1-\alpha}(\measure^\varphi_{\tcount}))= \adjfunc{\tau_{x}}(\q{1-\alpha}(\measure_{\tcount}))$ showing that, the prediction set does not depend on the choice of $\varphi$. To illustrate the advantage of our method, in~\Cref{fig:funnel} we present the prediction sets obtained with the classical CP method and \algo\ in the case of a Neal's funnel-shaped distribution in 2 dimensions; see~\citet[Section~9]{neal2003slice}.
%%%%%%%%%%%%%%%%%%%%%%%%%%%%%%%%%%%%%%%%%%%%%%%%

The natural approach for the general case is to use the conditional distribution $\Pi_{Y\mid X=x}$ or its estimate to find Highest Predictive Density (HPD) regions and calibrate their size with the help of \algo. We develop the respective general algorithm \hpdalgo\ in Appendix~\ref{sec:hpd}. However, the procedures to find HPDs are usually highly non-trivial and we only investigate this approach experimentally for synthetic data; see Section~\ref{sec:synthetic_data}.  Next, we provide a specific implementation of our general \algo\ framework that is universally applicable.

\paragraph{\algo\ with Implicit Conditional Generative Model: \pcpalgo.}
We also develop a second instance of the \algo\ algorithm, inspired by~\cite{wang2023probabilistic}. 
Unlike \hpdalgo, this approach does not require the conditional density. Instead, it is designed for cases where the conditional generative model (CGM) $\distr_{Y \mid X}$ is implicit, meaning we cannot evaluate it pointwise while being able to sample from it.
For each calibration point $X_k$, we draw $M$ random variables $\{\hat{Y}_{k,i}\}_{i=1}^{M}$ from $\distr_{Y \mid X=X_k}$.
We denote $Z_k=(\hat{Y}_{k,1},\ldots,\hat{Y}_{k,M})$ and consider the confidence sets as the union of spheres centered around the sample points \(\confReg_{Z_k}(X_k; t) = \cup_{i=1}^{M}\mathrm{B}(\hat{Y}_{k,i},t) \). With such choice, we get $\bar{\lambda}_{k}=\min_{i=1}^{M}\|Y_k-\hat{Y}_{k,i}\|$.
We then draw a second sample $\{\tilde{Y}_{k,j}\}_{j=1}^{\tilde{M}}$, and compute $\bar{\tau}_k=\{ t \in \R_+\colon \tilde{M}^{-1} \sum_{j=1}^{\tilde{M}} \1_{\confReg_{Z_k}(X_k;\adjfunc{t}(\varphi))}(\tilde{Y}_{k,j}) \ge 1 - \alpha\}$.
It is easily seen that
\begin{equation*}
  \textstyle
  \bar{\tau}_k
  = \pr{t\mapsto \adjfunc{t}(\varphi)}^{-1} \ac{\q{1-\alpha}\pr{
    \frac{1}{\tilde{M}} \sum_{j=1}^{\tilde{M}} \delta_{\min_{i=1}^{M}\|\tilde{Y}_{k,j}-\hat{Y}_{k,i}\|}
  }}.
\end{equation*}
Given a new input $X_{\tcount+1}\in\XC$, we sample $Z_{\tcount+1}=(\hat{Y}_{\tcount+1,1},\ldots,\hat{Y}_{\tcount+1,M})$ and obtain prediction set as follows
\begin{equation*}
  \textstyle
  \mathcal{C}_{\alpha}(X_{\tcount+1})
  = \ac{y\in\YC\colon \min_{i=1}^{M}\|y-\hat{Y}_{\tcount+1,i}\|\le \adjfunc{\bar{\tau}_{\tcount+1}}\pr{\q{1-\alpha}\pr{\measure_{\tcount}}}},
\end{equation*}
where $\measure_{\tcount}$ is given in~\eqref{eq:def:mu}.
The \pcpalgo\ method employs the same confidence set $\confReg_{z}(x;t)$ as the one used by \texttt{PCP}. This method effectively captures multimodalities using balls centered at likely outputs $\hat{Y}_{\tcount+1,i}$.
Furthermore, the conformity scores used by \texttt{PCP} correspond to our $\lambda_{x,y,z}$. However, the key distinction between the two algorithms lies in the additional parameter $\tau_{x,z}$ for \pcpalgo, which requires the generation of a second random sample from $\distr_{Y\mid X=x}$. This method is especially useful when solving equation~\eqref{eq:def:tau} is intractable. We summarize \pcpalgo\ in Algorithm~\ref{algo:CP2-PCP}.

\begin{algorithm}[t]
  \caption{\pcpalgo}
  \label{algo:CP2-PCP}
  \begin{algorithmic}[]
      \State {\bfseries Input:} dataset $\{(X_k,Y_k)\}_{k\in[\tcount]}$, significance level $\alpha$, conditional distribution $\distr_{Y \mid X}$, function $\adjfunc{t}$.
      \State \algorithmiccomment{Compute the $(1-\alpha)$-quantile}
      \For{$k=1$ {\bfseries to} $\tcount$}
          \State Sample $\{\hat{Y}_{k,i}\}_{i=1}^{M}$ and $\{\tilde{Y}_{k,j}\}_{j=1}^{\tilde{M}}$ from $\distr_{Y \mid X=X_k}$
          \State Set $\bar{\lambda}_{k} = \min_{i=1}^{M}\|Y_k-\hat{Y}_{k,i}\|$
          \State Set $\bar{\tau}_k = (t\mapsto\adjfunc{t}(\varphi))^{-1} \act{\q{1-\alpha}\prt{ \tilde{M}^{-1} \sum_{j=1}^{\tilde{M}} \delta_{\min_{i=1}^{M}\|\tilde{Y}_{k,j}-\hat{Y}_{k,i}\|}}}$
      \EndFor
      \State $\q{1-\alpha}\pr{\measure_{\tcount}} \gets \lceil (1-\alpha) (\tcount+1) \rceil$-th smallest value in $\{\adjfunc{\bar{\tau}_{k}}^{-1}(\bar{\lambda}_k)\}_{k\in[\tcount]}\cup\{\infty\}$
      \State \algorithmiccomment{Compute the prediction set for a new point $x\in\XC$}
      \State Sample $z=\{\hat{Y}_{i}\}_{i=1}^{M}$ and $\{\tilde{Y}_{j}\}_{j=1}^{\tilde{M}}$ from $\distr_{Y \mid X=x}$
      \State Set $\tau_{x,z} = \pr{t\mapsto \adjfunc{t}(\varphi)}^{-1} \act{\q{1-\alpha}\prt{ \tilde{M}^{-1} \sum_{j=1}^{\tilde{M}} \delta_{\min_{i=1}^{M}\|\tilde{Y}_{j}-\hat{Y}_{i}\|}}}$
      \State {\bfseries Output:} $\mathcal{C}_{\alpha}(x) = \cup_{i=1}^{M} \mathrm{B}(\hat{Y}_{i}, \adjfunc{\tau_{x,z}}\prt{\q{1-\alpha}\prt{\measure_{\tcount}}})$.
  \end{algorithmic}
\end{algorithm}

\section{Theoretical Guarantees}
\label{sec:theory}
% !TEX root = ../main.tex

In this section, we provide both marginal and conditional guarantees for the prediction set $\mathcal{C}_{\alpha}(x)$ given in~\eqref{eq:pred-set}.
The validity of these guarantees is ensured by the exchangeability of the calibration data, with the exception of \Cref{corr:coverage-conditional:asymptotic} which relies on a concentration inequality and thus requires i.i.d. calibration data.
The following theorem establishes marginal validity of the predictive set defined by \algo.
\thmspace

\begin{theorem}\label{thm:coverage:marginal}
  Assume \Cref{ass:confReg}-\Cref{ass:tau}.
  Then, for any $\alpha\in(0,1)$, it holds \(    1 - \alpha
    \le \prob\pr{Y_{\tcount+1} \in \mathcal{C}_{\alpha}(X_{\tcount+1})} \).
  Moreover, if the conformity scores $\{\adjfunc{\bar{\tau}_{k}}^{-1}(\bar{\lambda}_{k})\}_{k=1}^{\tcount+1}$ are almost surely distinct, then it also holds that  
 \(
    \prob\pr{Y_{\tcount+1} \in \mathcal{C}_{\alpha}(X_{\tcount+1})}
    < 1 - \alpha + (\tcount+1)^{-1}.
\)
\end{theorem}
The proof is postponed to \Cref{subsec:proof:marginal-conditional}. Moreover, the upper bound on the coverage always holds when the distribution of $\adjfunc{\bar{\tau}_{k}}^{-1}(\bar{\lambda}_{k})$ is continuous.
Now, we will investigate the conditional validity. 
Denote by $\tv$ the total variation distance and by $\prob^{\mathcal{T}}$ the conditional probability given the training data.
\thmspace

\begin{theorem}\label{thm:coverage:conditional}
  Assume \Cref{ass:confReg}-\Cref{ass:tau}, and let $\alpha\in(0,1)$.
  For any $x\in\R^d$ and $z\in\ZC$, it holds
  \begin{equation*}
    \prob^{\mathcal{T}}\pr{Y_{\tcount+1} \in \mathcal{C}_{\alpha}(x) \mid (X_{\tcount+1}, Z_{\tcount+1})= (x,z)}
    \ge 1 - \alpha - \tv(\textup{P}_{Y\mid X=x}; \distr_{Y \mid X=x})
    - p_{\tcount+1}^{(x,z)},
  \end{equation*}
  where
  $
    p_{\tcount+1}^{(x,z)} = \prob^{\mathcal{T}}\pr{\q{1-\alpha}(\measure_{\tcount}) < \adjfunc{\bar{\tau}_{\tcount+1}}^{-1}(\bar{\lambda}_{\tcount+1}) \le \varphi \,\vert\, (X_{\tcount+1},Z_{\tcount+1})=(x,z)}
  $.
\end{theorem}
The proof is postponed to \Cref{subsec:proof:marginal-conditional}.
The more accurately the estimator $\distr_{Y \mid X=x}$ approximates the true conditional distribution, the closer the result will be to $1-\alpha$.  The second term in the lower bound is $p_{\tcount+1}^{(x,z)}$. Its expected value is upper bounded by $\E[p_{\tcount+1}^{(X,Z)}]\le \alpha$, and non-asymptotic bounds for this error term are developed in~\Cref{subsect:proof:cond-coverage}.

We will now briefly discuss the asymptotic conditional coverage guarantee; details are provided in the supplementary paper. Assuming the availability of an oracle for the predictive distribution, i.e., $\textup{P}_{Y\mid X=x} = \distr_{Y \mid X=x}$,we get under \Cref{ass:confReg} and \Cref{ass:tau}, that for any $t \in \mathbb{R}$,
\begin{align*}
  \prob\pr{\lambda_{X,Y,Z}\le \adjfunc{\tau_{X,Z}}(t)\,\vert\, X=x, Z=z}
  &= \prob\pr{Y\in \confReg_{z}(x; \adjfunc{\tau_{x,z}}(t))\,\vert\, X=x, Z=z}
  \\
  &= \distr_{Y \mid X=x}\pr{\confReg_{z}(x; \adjfunc{\tau_{x,z}}(t))},
\end{align*}
where $(X,Y,Z)$ follows the same distribution than $(X_k,Y_k,Z_k)$, $k \in \{1,\dots,n\}$.
Note that $\distr_{Y \mid X=x}(\confReg_{z}(x; \adjfunc{\tau_{x,z}}(t)))\ge 1-\alpha$ if and only if $t\ge \varphi$, which implies that
\begin{equation}\label{eq:prob-lambdaXY}
  \text{$\prob(\adjfunc{\tau_{X,Z}}^{-1}(\lambda_{X,Y,Z}) \le t \mid (X,Z)=(x,z)) \ge 1-\alpha$ \, if and only if \, $t\ge \varphi$.}
\end{equation}
From~\eqref{eq:prob-lambdaXY} it is easily seen that the $(1-\alpha)$-quantile of $\adjfunc{\tau_{X,Z}}^{-1}(\lambda_{X,Y,Z})$ is $\varphi$.
The Glivenko--Cantelli Theorem~\cite[Theorem 19.1]{van2000asymptotic} demonstrates that $\sup_{t\in\R}|\measure_{\tcount}(-\infty,t] - \prob(\adjfunc{\tau_{X,Z}}^{-1}(\lambda_{X,Y,Z})\le t)|\to 0$ almost surely as $\tcount\to\infty$, where $\measure_{\tcount}$ is defined in~\eqref{eq:def:mu}.
Since the convergence of the c.d.f. implies the convergence of the quantile function~\cite[Lemma 21.2]{van2000asymptotic}, we deduce that $\q{1-\alpha}(\measure_{\tcount})\to \varphi$ almost-surely as $\tcount\to\infty$.
Under weak additional conditions this implies that  $\lim_{\tcount\to\infty} p_{\tcount+1}^{(x,z)}=0$, $\bar{\distr}_{Z\mid X} \times \textup{P}_{X}$-almost everywhere, where $\bar{\distr}_{Z \mid X}$ is the Markov kernel used to draw the auxiliary variables $z$; see~\Cref{subsect:proof:quantile-rewritten}. In this case, \Cref{thm:coverage:marginal} implies the asymptotic validity of \algo.

In practice, the oracle is unavailable. In the following theorem, we examine the asymptotic conditional conformal validity as the size of the training dataset, 
$m_{\tcount}$, goes to infinity with $\tcount$. In most cases, $\lim_{\tcount \to \infty} m_{\tcount}/\tcount = \gamma >0$, but this is not required here.
To make the dependency of the estimator on the size of the training set explicit, we will denote the conditional distribution a $\distr_{Y \mid X}^{(m_{\tcount})}$.
Consider the following assumption.
\begin{assumption}\label{ass:tv:cv-prob}
    There exists sequence $(r_{\tcount})$ such that
    $
      \lim\limits_{\tcount\to\infty} \prob(
          \tv(\textup{P}_{X,Y}; \textup{P}_{X}\times\distr_{Y \mid X}^{(m_{\tcount})})
          \le r_{\tcount}
      )
      = 1.
    $
\end{assumption}
In most interesting case, we have $\lim_{\tcount \to \infty} r_{\tcount}= 0$. Such types of bounds can be deduced from \citet[Chapter~9]{devroye2001combinatorial}.
Let $(X,Y,Z)$ and $(X,\hat{Y},Z)$ be random variables distributed according to $\textup{P}_{X,Y}\times \bar{\distr}_{Z\mid X}$ and $\textup{P}_{X}\times \distr_{Y\mid X}^{(m_{\tcount})}\times \bar{\distr}_{Z\mid X}$, respectively.
\thmspace

%%%%%%%%%%%%%%%%%%% REMETTRE %%%%%%%%%%%%%%%%%%%
% \begin{theorem}\label{corr:coverage-conditional:asymptotic}
%   Assume \Cref{ass:confReg}-\Cref{ass:tau}-\Cref{ass:tv:cv-prob} hold.
%   If the distributions of $\adjfunc{\tau_{X,Z}}^{-1}(\lambda_{X,Y,Z})$ and $\adjfunc{\tau_{X,Z}}^{-1}(\lambda_{X,\hat{Y},Z})$ are continuous, then, $\forall \epsilon\in(0,1)$ there exists $(\Lambda_{\tcount}^{(\epsilon)})_{\tcount\in\N}$ such that $\liminf_{\tcount\to\infty}\prob((X_{\tcount+1},Z_{\tcount+1}) \in \Lambda_{\tcount}^{(\epsilon)}) \ge 1 - \epsilon$ and also
%   \begin{equation*}
%       \sup_{(x,z)\in\Lambda_{\tcount}^{(\epsilon)}} \abs{ \prob^{\mathcal{T}}\pr{Y_{\tcount+1}\in\mathcal{C}_{\alpha}(X_{\tcount+1}) \,\vert\,( X_{\tcount+1},Z_{\tcount+1})=(x,z)}
%       - 1 + \alpha }
%       = \Oh_{\prob}\pr{ \sqrt{\tcount^{-1}\log\tcount} + r_{\tcount}}
%       .
%   \end{equation*}
% \end{theorem}
%%%%%%%%%%%%%%%%%%%%%%%%%%%%%%%%%%%%%%%%%%%%%%%
\begin{theorem}\label{corr:coverage-conditional:asymptotic}
  Assume \Cref{ass:confReg}-\Cref{ass:tau}-\Cref{ass:tv:cv-prob} hold.
  If the distributions of $\adjfunc{\tau_{X,Z}}^{-1}(\lambda_{X,Y,Z})$ and $\adjfunc{\tau_{X,Z}}^{-1}(\lambda_{X,\hat{Y},Z})$ are continuous, then, it holds
  \begin{equation*}
      \abs{ \prob^{\mathcal{T}}\pr{Y_{\tcount+1}\in\mathcal{C}_{\alpha}(X_{\tcount+1}) \,\vert\, X_{\tcount+1},Z_{\tcount+1}}
      - 1 + \alpha }
      = \Oh_{\prob}\pr{ \sqrt{\tcount^{-1}\log\tcount} + r_{\tcount}}
      .
  \end{equation*}
\end{theorem}
%%%%%%%%%%%%%%%%%%%%%%%%%%%%%%%%%%%%%%%%%%%%%%%%

In~\citet{lei2018distribution,izbicki2019flexible,sesia2020comparison}, the asymptotic conditional validity is demonstrated by assuming the consistency of their methods' estimators. For instance, \citet{romano2019conformalized} assume that the conditional quantile regressor converges in $L^2$ towards the true quantile with high probability.

\section{Numerical Experiments}
\label{sec:expriments}
% !TEX root = ../main.tex

  \begin{figure}[t!]
    \centering
    \begin{subfigure}{0.305\textwidth}
      \includegraphics[width=\textwidth]{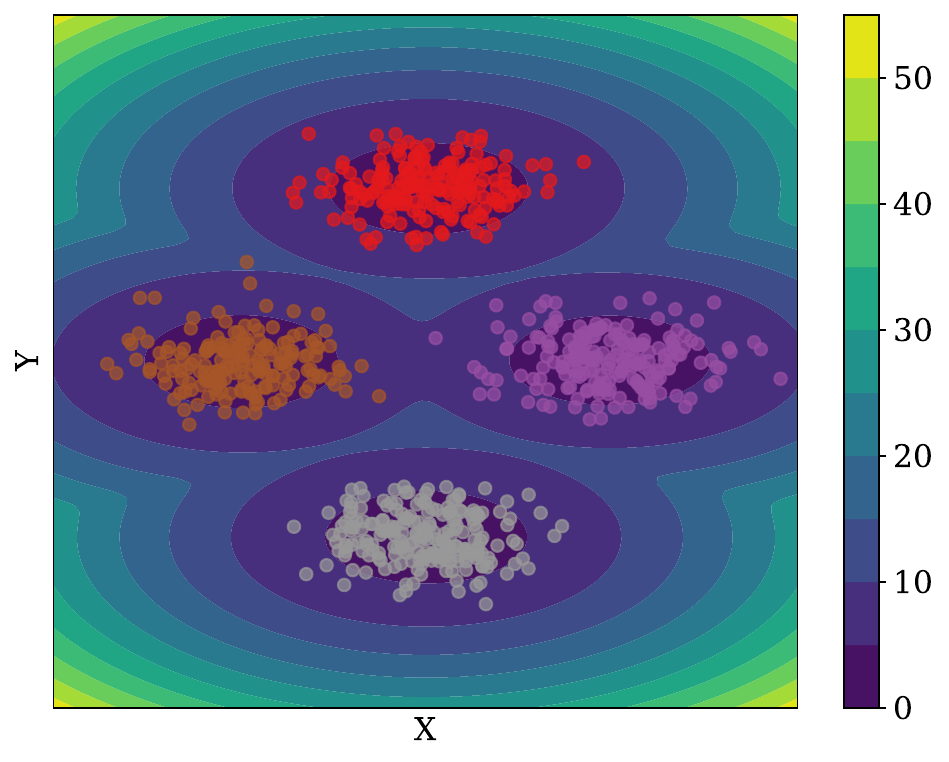}
      \caption{Data and MDN estimate}
      \label{fig:gmm:data}
    \end{subfigure}
    % ~~~~
    \begin{subfigure}{0.334\textwidth}
      \includegraphics[width=\textwidth]{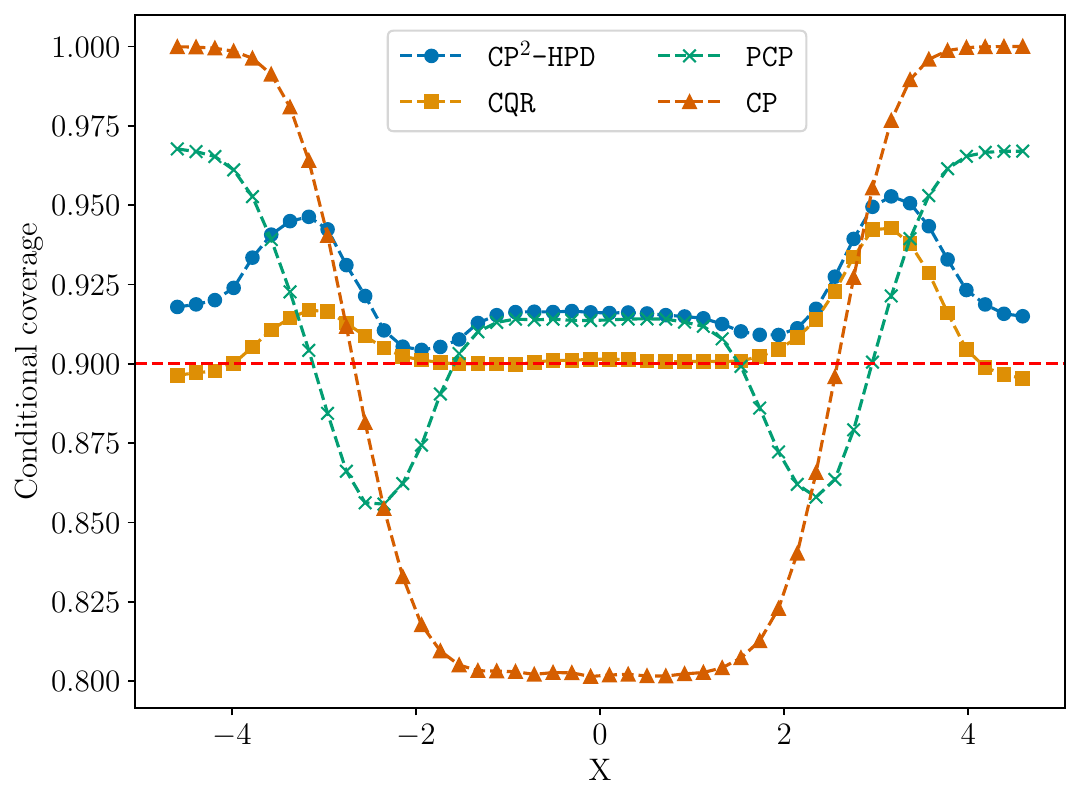}
      \caption{Conditional coverage}
      \label{fig:gmm:cc}
    \end{subfigure}
    % ~~~~
    \begin{subfigure}{0.32\textwidth}
      \includegraphics[width=\textwidth]{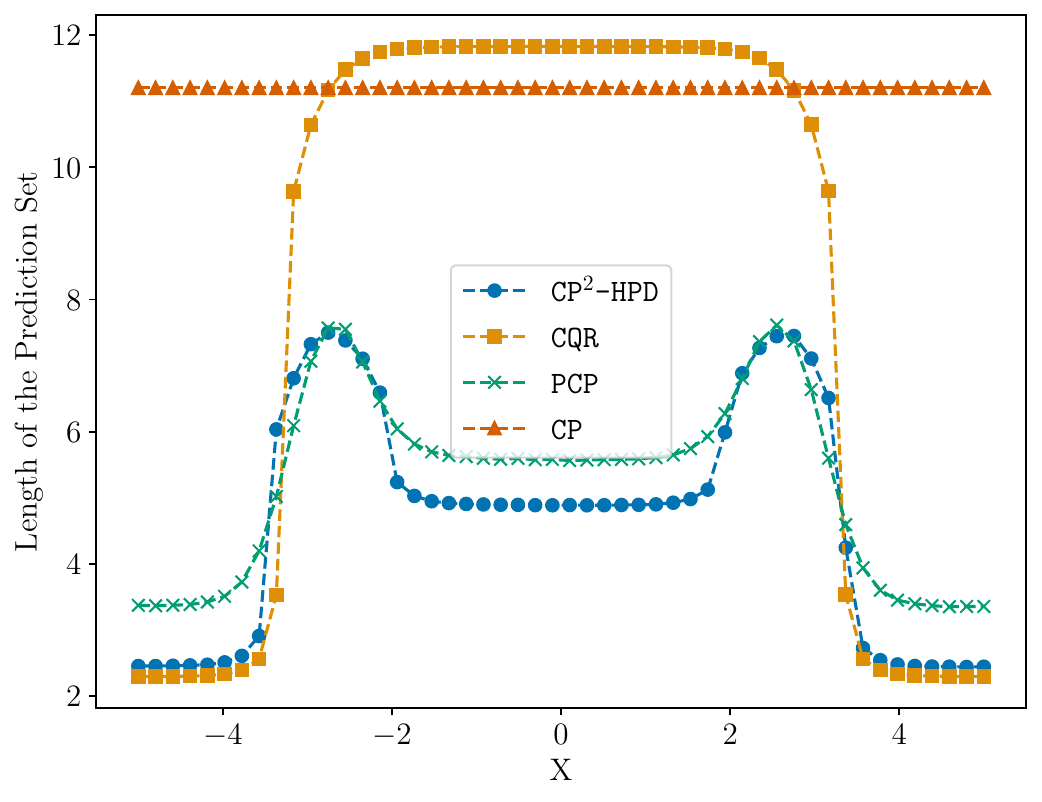}
      \caption{Length of the prediction set}
      \label{fig:gmm:cc:length}
    \end{subfigure}
    \caption{Mixture Density Network: the multimodal case.}
    \label{fig:gmm}
  \end{figure}
  In this section, we conduct a comprehensive analysis demonstrating the advantage  of \algo\ compared to standard and adaptive split conformal algorithms. Specifically, we benchmark our algorithm against several state-of-the-art methods: Conformalized Quantile Regression~\citep{romano2019conformalized}, Conformalized Histogram Regression~\citep{sesia2021conformal} and Probabilistic Conformal Prediction~\citep{wang2023probabilistic}. All these method share some key aspects: they are built on top of the pre-trained models and do not require access to training data or the model's internals on both calibration and prediction steps. We aim to answer these specific questions: how does \algo\ performs in terms of coverage, conditional coverage and predictive set volume when compared to state-of-the-art methods on synthetic and real data.

\subsection{Synthetic Data Experiment}
\label{sec:synthetic_data}
  In this example, $(X_k,Y_k)$ is sampled from a mixture of $P=4$ Gaussians; see \Cref{fig:gmm:data}. %The results for other classical 2-d datasets lead to similar conclusions.
  The number of training and calibration samples is $m = 10^4$ and $\tcount=10^3$, respectively. We fit a Mixture Density Network (MDN) as an explicit generative model, $\pdfdistr_{Y\mid X=x}(y) = \sum_{\ell=1}^P \pi_\ell(x) \gauss(y; \measure_\ell(x), \sigma^2_\ell(x))$, where $\measure_\ell(\cdot)$, $\sigma_\ell(\cdot)$ and $\pi_\ell(\cdot)$ are all modeled by fully connected 2-layers neural networks (the condition $\sum_{\ell=1}^P \pi_\ell(x)=1$ is ensured by using softmax activation functions).
  We use \hpdalgo\ (the calculation of the HPD rates as well as $\tau_x$ and $\lambda_{x,y}$ is explicit in this case) with \( f_{t}(v) = t v \). 
  The parameters of the MDN are trained by maximizing the likelihood on the training set.

  We compare the plain \hpdalgo, \texttt{PCP} (with the same MDN as \hpdalgo\ and $M=50$ draws) and \texttt{CQR}. All methods achieve the desired marginal coverage $1-\alpha=0.9$. We illustrate the conditional coverage in \Cref{fig:gmm:cc} and the lengths of the predictive sets in \Cref{fig:gmm:cc:length}. 
  CP with a fixed-width predictive set performs poorly in this multimodal example, both in terms of the size of the confidence set and the conditional coverage.
  \hpdalgo\ and \texttt{CQR} perform similarly in terms of conditional coverage (which remains close to $1-\alpha=0.9$).
  The conditional coverage of \texttt{PCP} varies between 0.85 and 0.95.
  \hpdalgo\ produces shorter prediction sets compared to \texttt{CQR} and \texttt{PCP}. This is because \hpdalgo\ uses an HPD confidence set that is more suitable for multimodal applications than the interval produced by \texttt{CQR}.

  \begin{figure}[t!]
    \centering
    \includegraphics[width=0.9\textwidth]{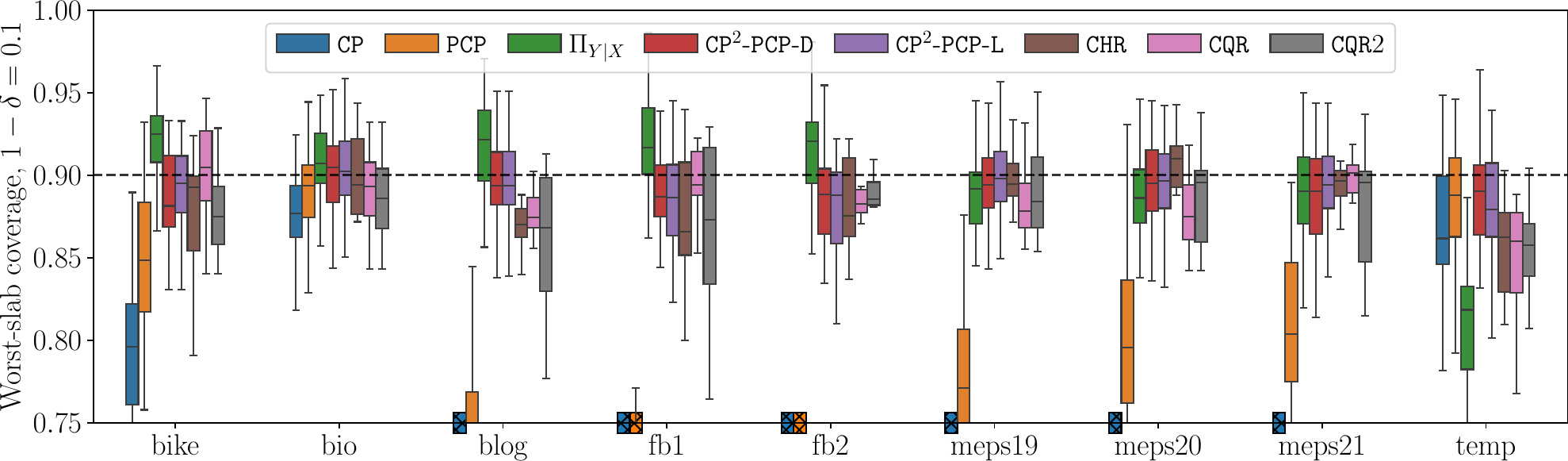}
    \caption{Worst-slab coverage on real data. Results averaged over 50 random splits of each dataset. Calibration and test set sizes set to 2000, 50 conditional samples for \texttt{PCP}, \algo\ and $\gammayx$. Worst-slab coverage parameter $(1-\delta)=0.1$. Nominal coverage level is $(1-\alpha)=0.9$ and is shown in dashed black. Methods with conditional coverage below $0.75$ shown as cross-hatched on horizontal axis.}
    \label{fig:real_wsc}
  \end{figure}

  % \bgroup\color{blue}
  % \begin{verbatim}

  % This toy example illustrates the behavior of some very typical conformal methods:

  % - **CQR Method:**
  %   - The CQR method constructs a prediction set for the output \( Y \) with values in \( \mathbb{R} \). This method generates a confidence interval and is therefore not suited to bimodal distributions of \( Y|X \). As a result, the associated prediction lengths are too wide.

  % - **PCP Method:**
  %   - The PCP method constructs prediction sets as the union of many balls. Unlike the CQR method, this approach has the advantage of tackling multimodality, as the generated prediction sets can contain several disjoint regions. However, the balls constructed have a fixed radius. In some areas of the space, this radius may be too small or too large, leading to under-coverage or over-coverage. We observed this issue in our synthetic data experiments. Specifically, we noted that the heteroscedasticity problem worsens for PCP when the number of balls is very large. This highlights the need for a more adaptive method that can adjust the size of the prediction sets based on the local variability of the data.

  % - **Our Method CP2-HPD:**
  %   - Our method CP2-HPD allows us to readjust the size of the radius according to the covariate \( x \) considered. Thus, we benefit from the PCP technique to deal with multimodality, and moreover, our mechanism resolves the problem of heteroscedasticity.
    
  % \end{verbatim}
  % \egroup

  \begin{figure}[t!]
    \centering
    \includegraphics[width=0.9\textwidth]{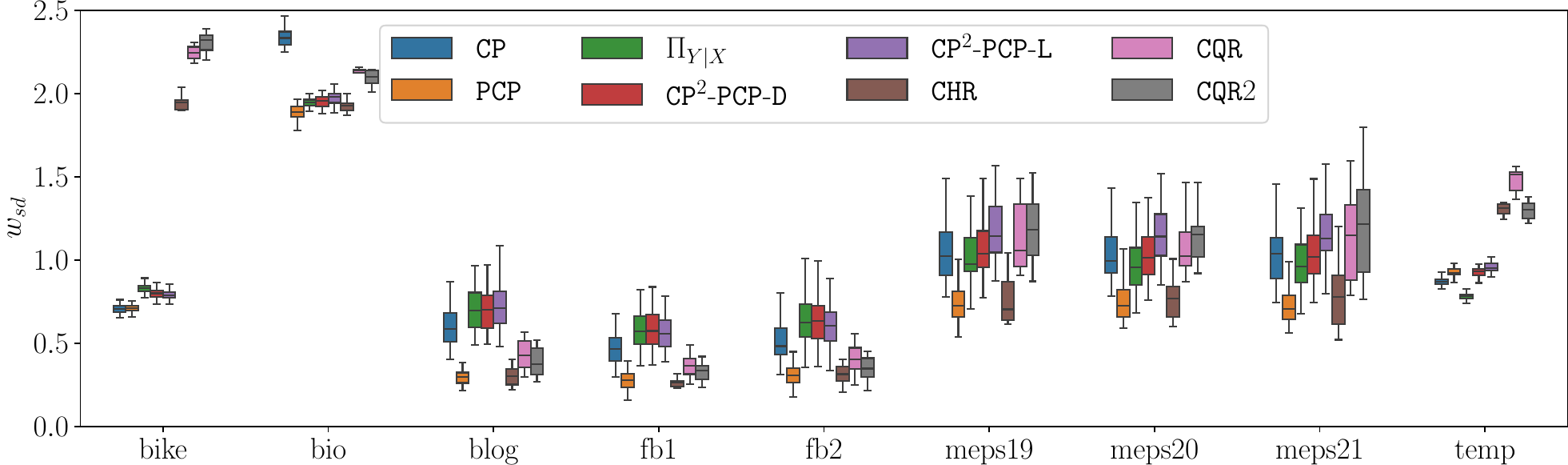}
    \caption{Sizes of the prediction sets on real data. We divide the size of the set by the standard deviation of response to present the results on the same scale.}
    \label{fig:real_size}
    \vskip -1em
  \end{figure}

\subsection{Real-world Regression Data Experiments}
  In this section, we study the performance of \pcpalgo\ on several real world regression datasets. In particular, we used the same datasets as in~\citet{wang2023probabilistic} to facilitate comparisons.

%  \bgroup\color{blue}
%
%    - Table 4 may provide more insight than Table 3. Conditional coverage (worst-slab coverage) in Table 3 is calculated over slabs containing at least 40\% of the data. In this scenario, most methods provide good coverage since the marginal coverage is achieved by all methods and the training data is sufficient.
%    - For Table 4, we have used slabs containing at least 10\% of the data, which should provide a more adequate estimate of the conditional coverage. According to Table 4, competing methods often fail to achieve \( 1-\alpha \) conditional coverage. For example, PCP drops to as low as 0.61 on the Facebook2 dataset. Conformalized Histogram Regression remains a strong baseline but still shows lower conditional coverage than all CP2-family methods on some datasets.
%
%  \egroup

\textbf{Datasets.} 
  We use publicly available regression datasets, which are also considered in~\citet{romano2019conformalized,wang2023probabilistic}. Some of them come from the UCI repository: bike sharing (\texttt{bike}), protein structure (\texttt{bio}), blog feedback (\texttt{blog}), Facebook comments (\texttt{fb1} and \texttt{fb2}). Other datasets come from US Department of Health surveys (\texttt{meps19}, \texttt{meps20} and \texttt{meps21}), and from weather forecasts (\texttt{temp})~\citep{tempdata2020}.

\textbf{Methods.}
  %We consider two main groups of algorithms: conformal-based methods and quantile regression-based methods. 
  We compare the proposed \pcpalgo\ method with Probabilistic Conformal Prediction (\texttt{PCP}; \citet{wang2023probabilistic}), Conformalized Quantile Regression (\texttt{CQR}; \citet{romano2019conformalized}) and Conformalized Histogram Regression (\texttt{CHR}; \citet{sesia2021conformal}). %\texttt{CQR} and \texttt{CHR} are based on the same internal neural network model.
  We also consider \texttt{CQR}$2$ which is a modification of \texttt{CQR} that uses inverse quantile nonconformity score. For our method and \texttt{PCP} we use a Mixture Density Network~\citet{bishop1994mixture} to estimate the conditional distribution $\textup{P}_{Y\mid X}$, since it was chosen in~\citet{wang2023probabilistic} as best-performing. We also consider different choices of $\adjfunc{t}$ for our method: \pcpalgo-\texttt{L} stands for \pcpalgo\ with $\adjfunc{t}(v) = tv$ and \pcpalgo-\texttt{D} stands for \pcpalgo\ with $\adjfunc{t}(v) = t + v$. Our implementation of \pcpalgo\ is summarized in~\Cref{algo:CP2-PCP}. 
  Additionally, we consider $\distr_{Y \mid X}$ which is a special case of \pcpalgo\ with $\adjfunc{t}(v)=t$. %, while \texttt{PCP} can be obtained with $\adjfunc{t}(v)=\lambda$. \vincent{Is it really pertinent?}

\textbf{Metrics.} 
  Empirical coverage (marginal and conditional) is the main quantity of interest for prediction sets. %\pcpalgo\ satisfies the marginal coverage guarantee but also achieves (a version of) conditional coverage.
  We evaluate worst-slab conditional coverage~\citep{cauchois2020knowing,romano2020classification} in our experiments, see details in Appendix~\ref{appendix:worst_slab}. We also measure the total size of the predicted sets, scaled by the standard deviation of the response $Y$.

\textbf{Experimental setup.} 
Our experimental setup largely follows the approach outlined in~\citet{wang2023probabilistic}. Specifically, we split each dataset into training, calibration, and testing sets. A Mixture Density Network (MDN) with 10 components is then trained to approximate the conditional distribution $\textup{P}_{Y \mid X}$. For each calibration and test point, we first compute the Gaussian Mixture parameters, forming $\distr_{Y \mid X}$, and subsequently draw $M = 5, 20, 50$ samples from these distributions, which yield $\mathcal{R}_z(x, t)$. This process is repeated across 50 different random splits of each dataset.
  \textbf{Results} of the experiments for $M=50$ samples are presented in Figures~\ref{fig:real_wsc} and~\ref{fig:real_size}, additional results are available in Appendix~\ref{suppl:expriments}. In terms of marginal coverage, all methods achieve the target $1-\alpha$ value, except for $\distr_{Y\mid X}$. %This demonstrates that estimating the conditional distribution is a challenging problem and the resulting estimates need to be further calibrated. \vincent{Do we let this sentence?}

  Standard conformal prediction fails to maintain the conditional coverage as expected. We can also observe that \texttt{PCP} consistently struggles with conditional coverage. On all the datasets \pcpalgo\ provides valid conditional coverage, while \texttt{CQR} fails on \texttt{blog} and \texttt{temp}. \texttt{CHR} method shows unstable performance not achieving conditional coverage more often than other methods but sometimes providing narrower predictions sets. Additionally, \pcpalgo\ significantly outperforms quantile regression-based methods in terms of size of the prediction sets on \texttt{bike}, \texttt{bio} and \texttt{temp} datasets.

  Finally, we assess conditional coverage with the help of clustering. We apply HDBSCAN~\citep{Campello2013DensityBasedCB, McInnes2017AcceleratedHD} method to cluster the test set and then compute coverage within clusters. Results for \texttt{fb1} dataset are presented in Figure~\ref{fig:ccov_f1_hdb}. We again observe that \texttt{CP} and \texttt{PCP} do not achieve conditional coverage and \texttt{CHR} and \texttt{CQR} performance is unstable. \pcpalgo\ on the other hand maintains valid conditional coverage on all clusters and even on outliers (cluster label \texttt{-1}). Note that these are all outliers combined and they may not lie in the same region of the input space.
  
  \begin{figure}[t!]
    \centering
    \includegraphics[width=0.7\textwidth]{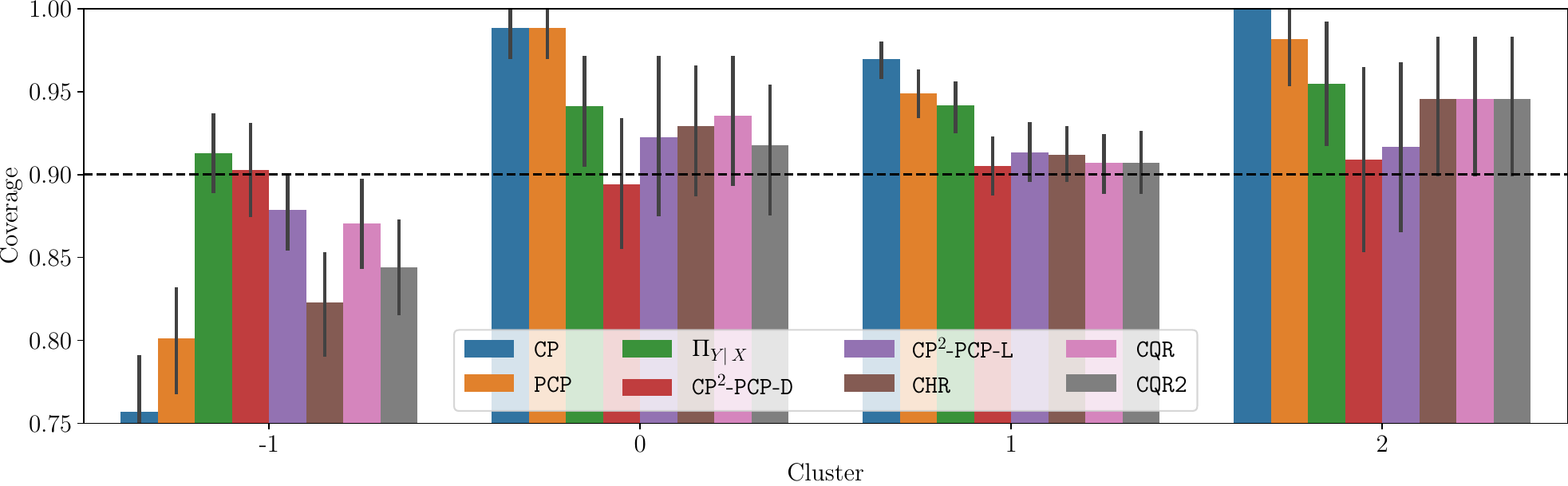}
    \caption{Conditional coverage for different clusters, \texttt{fb1} dataset. We have used HDBSCAN algorithm with minimum cluster size of 100, \texttt{min\_samples} hyper-parameter of 20 and $l_2$ metric. Cluster label \texttt{-1} corresponds to the outliers. Sample size for sampling-based methods was set to 50. Nominal coverage equals $(1-\alpha)=0.9$ and is shown in dashed blacks. }
    \label{fig:ccov_f1_hdb}
  \end{figure}

  \begin{figure}[t!]
    \centering
    \includegraphics[width=0.9\textwidth]{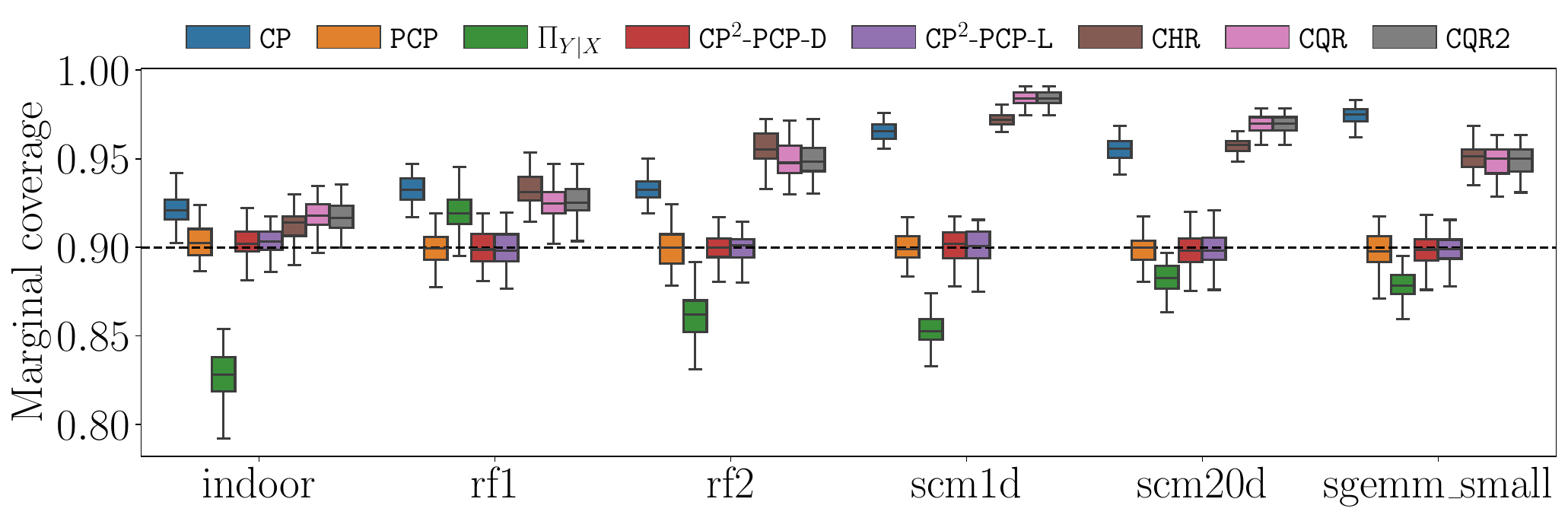}
    \caption{Marginal coverage for multi-target datasets, 50 replications. Sample size was set to 1000. Nominal coverage equals $(1-\alpha)=0.9$ and is shown by dashed black line.}
    \label{fig:md_cov}
  \end{figure}

\subsection{Real-world Regression Data with Multi-dimensional Targets}
  We also study \texttt{CP2} family of algorithms on the multi-target regression problems. Since selecting the threshold $\tau$ for our methods is not dependent on the number of dimensions in $Y$ their application is straightforward. On the other hand, most other methods are inherently one-dimensional thus require the use of the Bonferroni correction~\citep{multiple_coomparisons_1961}. Each coordinate is treated independently with miscoverage level adjusted to $\alpha/d$, where $d$ is the number of targets. As a result, for quantile regression-based methods prediction sets are rectangular cuboids, formed as a product of the corresponding intervals.

  \textbf{Datasets.} We consider open-source multidimensional regression datasets: river flow data \texttt{rf1} and \texttt{rf2}~\citep{Xioufis2012MultitargetRV}, supply chain management \texttt{scm1d} and \texttt{scm20d}~\citep{Xioufis2012MultitargetRV}, indoor localisation \texttt{indoor}~\citep{UJIIndoorLoc}, GPU computation time \texttt{sgemm\_small}\footnote{The full dataset contains 241600 examples. Due to computational constraints we randomly subsample 10000 examples for each replication of our experiment.}~\citep{BallesterRipoll2017SobolTT}.

  We use the same \textbf{metrics} as before: marginal coverage and worst-slab coverage. Evaluating the difference in prediction set size is more complex in case of multiple dimensions. Due to computational constraints we perform pairwise comparisons between our methods and selected baselines, measuring approximate areas of 2D projections of the prediction sets ~\citep{wang2023probabilistic}. These results can be found on Figure~\ref{fig:md_2d}. We approximate areas using a grid and fewer samples.

  Since our methods naturally extend beyond one dimension, the \textbf{experimental setup} is almost identical. We use the same underlying model for $\textup{P}_{Y\mid X}$, the prediction set is now a union of $d$-dimensional balls of the same radius around the sampled centers. The number of samples is increased to 1000. 

  \begin{figure}[t!]
    \centering
    \includegraphics[width=0.8\textwidth]{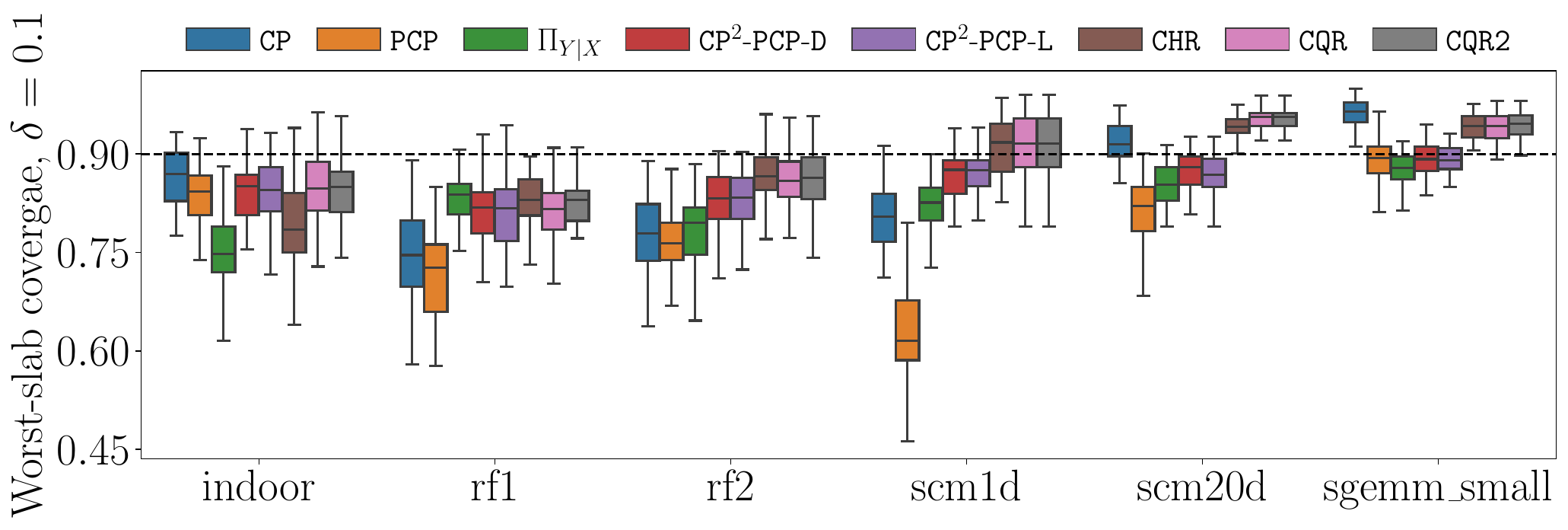}
    \caption{Conditional coverage for multi-target datasets targets, 50 replications. Sample size was set to 1000. Nominal coverage equals $(1-\alpha)=0.9$ and is shown in dashed black. Worst-slab coverage parameter $(1-\delta)=0.1$.}
    \label{fig:md_wsc}
  \end{figure}

  \begin{figure}[t!]
    \centering \includegraphics[width=0.8\textwidth]{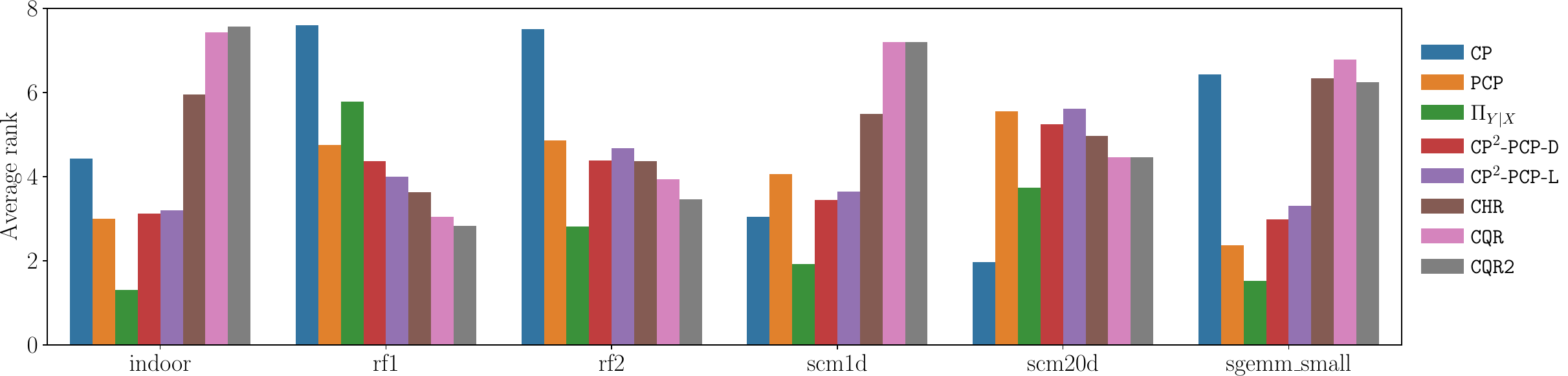}
    \caption{Average rank of the projected set size. For each pair of targets area of the corresponding 2D projection of the prediction set is calculated. For each test point and each pair of targets methods are ranked. Lower rank is smaller area. This graph shows averaged results of 10 replications.}
    \label{fig:md_2d}
  \end{figure}

  \textbf{Results.} In Figure~\ref{fig:md_cov} we show marginal coverage attained by different algorithms. As expected, naive application of 1D techniques CQR, CQR2 and CHR to multiple outputs produces significant overcover. \texttt{PCP} and \algo\ methods naturally extend to multidimensional targets and provide correct marginal coverage.
  
  In Figure~\ref{fig:md_wsc} we present the conditional coverage estimates for multi-target datasets. \texttt{PCP} significantly undercovers on \texttt{rf1}, \texttt{rf2} and \texttt{scm1d} datasets, while \algo\ comes very close to the nominal coverage of $0.9$. In case of CQR, CQR2 and CHR, they still overcover (\texttt{scm20d}, \texttt{sgemm}) or perform comparably to our approach.
  
  Figure~\ref{fig:md_2d} shows the aggregated results of the set size comparisons in multidimensional target setting. For each test point and each pair of axes we rank the methods by the area of the projection of the corresponding prediction set. The plot show average rank for each method, aggregated across all axes pairs and replications. Lower rank corresponds to smaller area, which is our goal. For datasets $\texttt{indoor}$, $\texttt{scm1d}$ and $\texttt{sgemm\_small}$ our approach performs better, while also providing sharper conditional covergae, as was shown earlier. On the remaining datasets \algo performs similarly to the competitors.

\section{Conclusion}
\label{sec:conclusion}
% !TEX root = ../main.tex

We address the challenge of conditional coverage in CP, and overcome previous negative results by assuming the knowledge of a good estimator of $\textup{P}_{Y\mid X}$.
Our proposed mechanism conformalized the conditional estimator $\distr_{Y \mid X}$ to ensure marginal validity while maintaining similar conditional coverage guarantees.
Specifically, if experts can provide an accurate conditional estimator, our algorithm \algo\ generates nearly conditionally valid multidimensional prediction sets. This approach offers a practical solution for tackling heteroscedasticity in various machine learning applications.

\subsubsection*{Acknowledgments}
  Part of this work has been carried out under the auspice of the Lagrange Mathematics and Computing Research Center. 
  E.M. is Funded by the European Union (ERC, Ocean, 101071601). Views and opinions expressed are however those of the author(s) only and do not necessarily reflect those of the European Union or the European Research Council Executive Agency. Neither the European Union nor the granting authority can be held responsible for them.

\bibliography{conformal}

%%%%%%%%%%%%%%%%%%%%%%%%%%%%%%%%%%%%%%%%%%%%%%%%%%%%%%%%%%%%

\appendix

\section{Additional Results and Calculations}
\label{suppl:proofs}
% !TEX root = ../main.tex

In this section, we analyze the theoretical results of~\Cref{sec:theory}.
First, let's recall the definition of the quantile function for any distribution $\measure_{\tcount}$ living in $\R$.
For any $\alpha\in(0,1)$, the quantile $\q{1-\alpha}(\measure_{\tcount})$ is defined by
\begin{equation*}
  \q{1-\alpha}(\measure_{\tcount})
  = \inf\ac{t\in\R \colon \measure_{\tcount}((-\infty, t]) \ge 1-\alpha}.
\end{equation*}
Given a measure $\distr_{Y \mid X=x}$ defined on $\sigma(\YC)$, we consider for all $x\in\XC$, $z\in\ZC$, the parameters $\tau_{x,z}$ and $\lambda_{x,y,z}$ given by
\begin{equation}\label{suppl:eq:def:tau-lambda}
  \begin{aligned}
    &\tau_{x,z} = \inf\ac{\tau \in \mathsf{T}\colon \distr_{Y \mid X=x}(\confReg_{z}(x; \adjfunc{\tau}(\varphi)))\ge 1-\alpha},
    \\
    &\lambda_{x,y,z} = \inf\ac{\lambda\in\mathsf{T}\colon y\in \confReg_{z}(x; \lambda)},
  \end{aligned}
\end{equation}
where $\varphi$ is chosen as in~\Cref{ass:tau}, and by convention we set $\inf\emptyset = \infty$.
We denote by $\delta_{v}$ the Dirac measure at $v\in\R$, and write $\bar{\tau}_{k} = \tau_{X_k,Z_k}$ and $\bar{\lambda}_{k} = \lambda_{X_k,Y_k,Z_k}$.
In this Appendix, we study the coverage of the prediction set given $\forall (x,z)\in\R\times\ZC$ by
\begin{equation*}
  \mathcal{C}_{\alpha}(x)
  = \confReg_{z}\pr{x; \adjfunc{\tau_{x,z}}\bigl(\q{1-\alpha}(\measure_{\tcount})\bigr)},
\end{equation*}
where the distribution $\measure_{\tcount}$ is defined as
\begin{equation*}
  \measure_{\tcount} = \frac{1}{\tcount+1}\sum_{k=1}^{\tcount} \delta_{\adjfunc{\bar{\tau}_{k}}^{-1}(\bar{\lambda}_{k})} + \frac{1}{\tcount+1} \delta_{\infty}.
\end{equation*}
The key idea behind the choice of $\bar{\tau}_{k}$ is to ensure that the conditional coverage of the prediction set $\mathcal{C}_{\alpha}(X_{k})$ is approximately $1-\alpha$ when the empirical distribution $\distr_{Y \mid X=X_k}$ is close to $\textup{P}_{Y\mid X=X_k}$. In other words, $\bar{\tau}_{k}$ is chosen such that the probability of the observed value $Y_k$ given $X_k$ falling inside the prediction set $\mathcal{C}_{\alpha}(X_k)$ is close to $1-\alpha$.
On the other hand, the parameter $\bar{\lambda}_{k}$ is used to ensure that the prediction set $\confReg_{Z_k}(X_k; \bar{\lambda}_{k})$ contains the observed value $Y_k$.
Moreover, note that $\bar{\tau}_{k}$ only depends on the input data $(X_k,Z_k)$, while $\bar{\lambda}_{k}$ depends on $(X_k,Y_k,Z_k)$. Thus, the i.i.d. property of $\{(X_k,Y_k,Z_k)\colon k\in[\tcount+1]\}$ ensures that the $\{(\bar{\tau}_{k},\bar{\lambda}_{k})\}_{k=1}^{\tcount+1}$ are also i.i.d.

\subsection{Proof of \Cref{thm:coverage:marginal,thm:coverage:conditional}}
\label{subsec:proof:marginal-conditional}
% !TEX root = ../main.tex

\begin{lemma}\label{lem:lambda}
  Assume \Cref{ass:confReg} hold. For any $(x,y,z)\in\XC\times\YC\times\ZC$, $\lambda_{x,y,z}$ exists in $\mathsf{T}$, and we have $y\in \confReg_{z}(x; \lambda_{x,y,z})$.
\end{lemma}
\begin{proof}
  Let $(x,y,z)\in\XC\times\YC\times\ZC$ be fixed. Since $\cap_{t\in\mathsf{T}}\confReg_{z}(x;t) = \emptysymbol$ and $\cup_{t\in\mathsf{T}}\confReg_{z}(x;t) = \YC$, we deduce the existence of $t_0$ and $t_1$ such that $y\notin\confReg_{z}(x;t_0)$ and $y\in\confReg_{z}(x;t_1)$. Therefore, $\{t\in\mathsf{T}\colon y\in \confReg_{z}(x; t)\}$ is non-empty and lower-bounded by $t_0$. Thus, the infimum $\lambda_{x,y,z}$ exists. 
  Now, let's prove that $y\in\confReg_{z}(x;\lambda_{x,y,z})$.
  Since $\lambda_{x,y,z} = \inf\{t\in\mathsf{T}\colon y\in \confReg_{z}(x; t)\}$, we deduce the existence of a decreasing sequence $\{\lambda_n\}_{n\in\N}$ such that $y\in \confReg_{z}(x; \lambda_n)$ and $\lim_{n\to\infty}\lambda_n=\lambda_{x,y,z}$. By definition of $\{\lambda_n\}_{n\in\N}$, we have $y\in\cap_{n\in\N}\confReg_{z}(x;\lambda_n)$. However, using~\Cref{ass:confReg}, remark that
  \begin{align*}
    \cap_{n\in\N}\confReg_{z}(x;\lambda_n) 
    &= \cap_{n\in\N}\cap_{t>\lambda_n}\confReg_{z}(x; t)
    \\
    &= \cap_{t>\lim\limits_{n\to\infty}\lambda_n}\confReg_{z}(x; t)
    \\
    &= \cap_{t>\lambda_{x,y,z}}\confReg_{z}(x; t)
    = \confReg_{z}(x; \lambda_{x,y,z}).
  \end{align*}
  Since $y\in \cap_{n\in\N}\confReg_{z}(x;\lambda_n)$, it implies that $y\in\confReg_{z}(x;\lambda_{x,y,z})$.
\end{proof}

We will now present the proof for~\Cref{thm:coverage:marginal}, which establishes the marginal validity of our proposed method.
\thmspace

\begin{theorem}\label{suppl:thm:coverage:marginal,thm:coverage:conditional}
  Assume \Cref{ass:confReg}-\Cref{ass:tau} hold, if $\{\adjfunc{\bar{\tau}_{k}}^{-1}(\bar{\lambda}_{k})\}_{k=1}^{\tcount+1}$ are almost surely distinct, then it follows
  \begin{equation}\label{eq:prob:YinC}
    1 - \alpha
    \le \prob^{\mathcal{T}}\pr{Y_{\tcount+1} \in \mathcal{C}_{\alpha}(X_{\tcount+1})}
    < 1 - \alpha + \frac{1}{\tcount+1}.
  \end{equation}
\end{theorem}

\begin{proof}
  Using~\Cref{lem:lambda}, we have
  \begin{align*}
    \nonumber
    \prob^{\mathcal{T}}\pr{Y_{\tcount+1} \in \mathcal{C}_{\alpha}(X_{\tcount+1})}
    &= \prob^{\mathcal{T}}\pr{Y_{\tcount+1} \in \confReg_{Z_{\tcount+1}}\pr{X_{\tcount+1}, \adjfunc{\bar{\tau}_{\tcount+1}}(\q{1-\alpha}(\measure_{\tcount}))}}
    \\
    &= \prob^{\mathcal{T}}\pr{\lambda_{\tcount+1} \le \adjfunc{\bar{\tau}_{\tcount+1}}(\q{1-\alpha}(\measure_{\tcount}))}.
  \end{align*}
  Since $\lambda \mapsto \adjfunc{\bar{\tau}_{\tcount+1}}(\lambda)$ is increasing by \Cref{ass:tau}, we deduce that
  \begin{equation*}
    \prob^{\mathcal{T}}\pr{\lambda_{\tcount+1} \le \adjfunc{\bar{\tau}_{\tcount+1}}(\q{1-\alpha}(\measure_{\tcount}))}
    = \prob^{\mathcal{T}}\pr{\adjfunc{\bar{\tau}_{\tcount+1}}^{-1}(\lambda_{\tcount+1}) \le \q{1-\alpha}(\measure_{\tcount})}.
  \end{equation*}
  Denote by $V_k = \adjfunc{\bar{\tau}_{k}}^{-1}(\bar{\lambda}_{k})$, the exchangeability of the data $\{(X_k,Y_k,Z_k)\colon k\in[\tcount+1]\}$ implies that
  \begin{multline*}
    \prob^{\mathcal{T}}\pr{V_{\tcount+1} \le \q{1-\alpha}\pr{\sum_{k=1}^{\tcount} \frac{\delta_{V_k}}{\tcount+1} + \frac{\delta_{\infty}}{\tcount+1}}}
    = \prob^{\mathcal{T}}\pr{V_{\tcount+1} \le \q{1-\alpha}\pr{\sum_{k=1}^{\tcount+1} \frac{\delta_{V_k}}{\tcount+1}}}
    \\
    = \frac{1}{\tcount+1} \sum_{k=1}^{\tcount+1} \E^{\mathcal{T}}\br{\1_{V_k} \le \q{1-\alpha}\pr{\frac{1}{\tcount+1}\sum_{k=1}^{\tcount+1} \delta_{V_k}}}
    \\
    = \E^{\mathcal{T}}\br{\E^{\mathcal{T}}\br{\1_{V_I} \le \q{1-\alpha}\pr{\frac{1}{\tcount+1}\sum_{k=1}^{\tcount+1} \delta_{V_k}} \,\bigg\vert\, V_1,\ldots,V_{\tcount+1}}},
  \end{multline*}
  where $I\sim\mathcal{U}nif(1,\ldots,\tcount+1)$. Therefore, the definition of the quantile function implies the lower bound in~\eqref{eq:prob:YinC}. Moreover, if there are no ties between the $\{V_k\}_{k=1}^{\tcount+1}$, then
  \begin{equation*}
    \prob^{\mathcal{T}}\pr{\adjfunc{\bar{\tau}_{\tcount+1}}^{-1}(\lambda_{\tcount+1}) \le \q{1-\alpha}(\measure_{\tcount})}
    < 1 - \alpha + \frac{1}{\tcount+1}.
  \end{equation*}
\end{proof}

The following lemma provides conditions under which $\distr_{Y \mid X=x}(\confReg_{z}(x;\adjfunc{\tau_{x,z}}(\varphi)))\ge 1-\alpha$.
\thmspace

\begin{lemma}\label{lem:tau}
  Assume \Cref{ass:confReg}-\Cref{ass:tau} hold, and let $\alpha\in(0,1$), $x\in\XC$, $z\in\ZC$.
  If $\distr_{Y \mid X=x}$ is a probability measure, then $\tau_{x,z}$ is defined in $\mathsf{T}$ and $\distr_{Y \mid X=x}(\confReg_{z}(x;\adjfunc{\tau_{x,z}}(\varphi)))\ge 1-\alpha$.
\end{lemma}
\begin{proof}
  Let $x\in\XC$ be such that $\distr_{Y \mid X=x}$ is a probability measure, and fix $z\in\ZC$. Since $\tau\mapsto \adjfunc{\tau}(\varphi)$ is increasing and bijective by~\Cref{ass:tau}, we have
  \begin{align*}
    \sup_{\tau\in\mathsf{T}} \distr_{Y \mid X=x}(\confReg_{z}(x; \adjfunc{\tau}(\varphi)))
    &= \distr_{Y \mid X=x}\pr{\cup_{\tau\in\mathsf{T}} \confReg_{z}(x; \adjfunc{\tau}(\varphi))}
    \\
    &= \distr_{Y \mid X=x}\pr{\cup_{t\in\mathsf{T}} \confReg_{z}(x; t)}
    = 1.
  \end{align*}
  The previous equality shows the existence of $\tau\in\mathsf{T}$ such that $\distr_{Y \mid X=x}(\confReg_{z}(x; \adjfunc{\tau}(\varphi)))\ge 1-\alpha$. Therefore $\{\tau \in \mathsf{T}\colon \distr_{Y \mid X=x}(\confReg_{z}(x; \adjfunc{\tau}(\varphi)))\ge 1-\alpha\}$ is non-empty. 
  This proves the existence of $\tau_{x,z} = \inf\{\tau \in \mathsf{T}\colon \distr_{Y \mid X=x}(\confReg_{z}(x; \adjfunc{\tau}(\varphi)))\ge 1-\alpha\}$ in $\mathsf{T}\cup\{-\infty\}$.
  Moreover, $\tau_{x,z}>-\infty$, otherwise we would have
  \begin{equation*}
    1 - \alpha
    \le \inf_{\tau\in\mathsf{T}} \distr_{Y \mid X=x}(\confReg_{z}(x; \adjfunc{\tau}(\varphi)))
    = \distr_{Y \mid X=x}\pr{\cap_{t\in\mathsf{T}} \confReg_{z}(x; t)}
    % = \distr_{Y \mid X=x}(\emptysymbol)
    = 0.
  \end{equation*}
  Therefore, we deduce that $\tau_{x,z}\in\mathsf{T}$.
  Lastly, remark that
  \begin{align*}
    \distr_{Y \mid X=x}(\confReg_{z}(x;\adjfunc{\tau_{x,z}}(\varphi)))
    &= \distr_{Y \mid X=x}(\cap_{\tau > \tau_{x,z}}\confReg_{z}(x;\adjfunc{\tau}(\varphi)))
    \\
    &= \inf_{\tau > \tau_{x,z}} \distr_{Y \mid X=x}(\confReg_{z}(x; \adjfunc{\tau}(\varphi)))
    \ge 1-\alpha.
  \end{align*}
\end{proof}

Now, we prove~\Cref{thm:coverage:conditional}. This result guarantees that the conditional confidence intervals constructed by our method approximately satisfy the desired coverage of $1-\alpha$.
Given $(x,y)\in\XC\times\ZC$, let's introduce
\begin{align*}
  &p_{\tcount+1}^{(x,z)} = \prob^{\mathcal{T}}\pr{\q{1-\alpha}(\measure_{\tcount}) < \adjfunc{\tau_{x,z}}^{-1}\pr{\lambda_{x,Y_{\tcount+1},z}} \le \varphi \,\vert\, X_{\tcount+1}=x, Z_{\tcount+1}=z},
  \\
  &q_{\tcount+1}^{(x,z)} = \prob^{\mathcal{T}}\pr{\varphi < \adjfunc{\tau_{x,z}}^{-1}\pr{\lambda_{x,Y_{\tcount+1},z}} \le \q{1-\alpha}(\measure_{\tcount}) \,\vert\, X_{\tcount+1}=x, Z_{\tcount+1}=z}.
\end{align*}
\thmspace

\begin{theorem}\label{supp:thm:coverage:conditional}
  Assume \Cref{ass:confReg}-\Cref{ass:tau} hold, let $x\in\R^d$ be such that $\distr_{Y \mid X=x}$ is a probability measure. For any $z\in\ZC$, it follows that
  \begin{multline*}
    1 - \alpha - \tv(\textup{P}_{Y\mid X=x}; \distr_{Y \mid X=x}) - p_{\tcount+1}^{(x,z)}
    \le \prob^{\mathcal{T}}\pr{Y_{\tcount+1} \in \mathcal{C}_{\alpha}(X_{\tcount+1}) \mid X_{\tcount+1}=x, Z_{\tcount+1}=z}
    \\
    \le \distr_{Y \mid X=x}(\confReg_{z}(x; \adjfunc{\tau_{x,z}}(\varphi))) + \tv(\textup{P}_{Y\mid X=x}; \distr_{Y \mid X=x}) + q_{\tcount+1}^{(x,z)}.
  \end{multline*}
\end{theorem}
\begin{proof}
  First, recall that $\mathcal{C}_{\alpha}(x)$ is given in~\eqref{eq:pred-set}, and $\lambda_{x,Y_{\tcount+1},z}$ is defined in~\eqref{suppl:eq:def:tau-lambda}. Applying \Cref{lem:lambda}, we know that $\lambda_{x,Y_{\tcount+1},z}$ is defined in $\mathsf{T}$, and also that $Y_{\tcount+1}\in \confReg_{z}(x;\lambda_{x,Y_{\tcount+1},z})$. Hence, it holds
  \begin{multline}\label{eq:eq:cov-cond:1}
    \prob^{\mathcal{T}}\pr{Y_{\tcount+1} \in \mathcal{C}_{\alpha}(X_{\tcount+1}) \,\vert\, X_{\tcount+1}=x, Z_{\tcount+1}=z}
    \\
    = \prob^{\mathcal{T}}\pr{Y_{\tcount+1} \in \confReg_{z}\pr{x; \adjfunc{\tau_{x,z}}(\q{1-\alpha}(\measure_{\tcount}))} \,\vert\, X_{\tcount+1}=x, Z_{\tcount+1}=z}
    \\
    = \prob^{\mathcal{T}}\pr{\lambda_{x,Y_{\tcount+1},z} \le \adjfunc{\tau_{x,z}}\pr{\q{1-\alpha}(\measure_{\tcount})} \,\vert\, X_{\tcount+1}=x, Z_{\tcount+1}=z}.
  \end{multline}
  Let's introduce the term $\prob^{\mathcal{T}}(\lambda_{x,Y_{\tcount+1},z} \le \adjfunc{\tau_{x,z}}(\varphi) \,\vert\, X_{\tcount+1}=x, Z_{\tcount+1}=z)$ as follows:
  \begin{multline}\label{eq:prob-diff}
    \prob^{\mathcal{T}}\pr{\lambda_{x,Y_{\tcount+1},z} \le \adjfunc{\tau_{x,z}}\pr{\q{1-\alpha}(\measure_{\tcount})} \,\vert\, X_{\tcount+1}=x, Z_{\tcount+1}=z}
    \\
    = \prob^{\mathcal{T}}\pr{\lambda_{x,Y_{\tcount+1},z} \le \adjfunc{\tau_{x,z}}\pr{\q{1-\alpha}(\measure_{\tcount})} \,\vert\, X_{\tcount+1}=x, Z_{\tcount+1}=z}
    \\
    \pm \prob^{\mathcal{T}}\pr{\lambda_{x,Y_{\tcount+1},z} \le \adjfunc{\tau_{x,z}}(\varphi) \,\vert\, X_{\tcount+1}=x, Z_{\tcount+1}=z}.
  \end{multline}
  Now, we will control the difference between the two terms of the previous equation.
  Let $A$ and $B$ be defined as
  \begin{align*}
    &A = \prob^{\mathcal{T}}\pr{\adjfunc{\tau_{x,z}}^{-1}(\lambda_{x,Y_{\tcount+1},z}) \le \q{1-\alpha}(\measure_{\tcount}) < \varphi \,\vert\, X_{\tcount+1}=x, Z_{\tcount+1}=z},
    \\
    &B = \prob^{\mathcal{T}}\pr{\adjfunc{\tau_{x,z}}^{-1}(\lambda_{x,Y_{\tcount+1},z}) \le \varphi \le \q{1-\alpha}(\measure_{\tcount}) \,\vert\, X_{\tcount+1}=x, Z_{\tcount+1}=z}.
  \end{align*}
  We have
  \begin{multline*}
    \prob^{\mathcal{T}}\pr{\lambda_{x,Y_{\tcount+1},z} \le \adjfunc{\tau_{x,z}}\pr{\q{1-\alpha}(\measure_{\tcount})} \,\vert\, X_{\tcount+1}=x, Z_{\tcount+1}=z}
    \\
    = A + B
    + \prob^{\mathcal{T}}\pr{\varphi < \adjfunc{\tau_{x,z}}^{-1}\pr{\lambda_{x,Y_{\tcount+1},z}} \le \q{1-\alpha}(\measure_{\tcount}) \,\vert\, X_{\tcount+1}=x, Z_{\tcount+1}=z},
  \end{multline*}
  and also
  \begin{multline*}
    \prob^{\mathcal{T}}\pr{\lambda_{x,Y_{\tcount+1},z} \le \adjfunc{\tau_{x,z}}(\varphi) \,\vert\, X_{\tcount+1}=x, Z_{\tcount+1}=z}
    \\
    = A + B
    + \prob^{\mathcal{T}}\pr{\q{1-\alpha}(\measure_{\tcount}) < \adjfunc{\tau_{x,z}}^{-1}\pr{\lambda_{x,Y_{\tcount+1},z}} \le \varphi \,\vert\, X_{\tcount+1}=x, Z_{\tcount+1}=z}.
  \end{multline*}
  Therefore, the difference between the terms introduced in~\eqref{eq:prob-diff} can be rewritten as
  \begin{multline}\label{eq:prob-diff:2}
    \prob^{\mathcal{T}}\pr{\lambda_{x,Y_{\tcount+1},z} \le \adjfunc{\tau_{x,z}}\pr{\q{1-\alpha}(\measure_{\tcount})} \,\vert\, X_{\tcount+1}=x, Z_{\tcount+1}=z}
    \\
    - \prob^{\mathcal{T}}\pr{\lambda_{x,Y_{\tcount+1},z} \le \adjfunc{\tau_{x,z}}(\varphi) \,\vert\, X_{\tcount+1}=x, Z_{\tcount+1}=z}
    \\
    = \prob^{\mathcal{T}}\pr{\varphi < \adjfunc{\tau_{x,z}}^{-1}\pr{\lambda_{x,Y_{\tcount+1},z}} \le \q{1-\alpha}(\measure_{\tcount}) \,\vert\, X_{\tcount+1}=x, Z_{\tcount+1}=z}
    \\
    - \prob^{\mathcal{T}}\pr{\q{1-\alpha}(\measure_{\tcount}) < \adjfunc{\tau_{x,z}}^{-1}\pr{\lambda_{x,Y_{\tcount+1},z}} \le \varphi \,\vert\, X_{\tcount+1}=x, Z_{\tcount+1}=z}.
  \end{multline}
  \begin{enumerate}[leftmargin=1em,labelsep=.5em]
    \item By definition of the total variation distance, we have
    \begin{multline*}
      \prob^{\mathcal{T}}\pr{\lambda_{x,Y_{\tcount+1},z} \le \adjfunc{\tau_{x,z}}(\varphi) \,\vert\, X_{\tcount+1}=x, Z_{\tcount+1}=z}
      \\
      \ge \prob^{\mathcal{T}}\pr{\lambda_{x,\hat{Y}_{\tcount+1},z}\le \adjfunc{\tau_{x,z}}(\varphi) \,\vert\, X_{\tcount+1}=x, Z_{\tcount+1}=z}
      - \tv(\textup{P}_{Y\mid X=x}; \distr_{Y \mid X=x}).
    \end{multline*}
    Moreover, \Cref{lem:tau} implies that
    \begin{multline*}
      \prob^{\mathcal{T}}\pr{\lambda_{x,\hat{Y}_{\tcount+1},z}\le \adjfunc{\tau_{x,z}}(\varphi) \,\vert\, X_{\tcount+1}=x, Z_{\tcount+1}=z}
      \\
      = \prob^{\mathcal{T}}\pr{\hat{Y}_{\tcount+1}\in\ac{y\in\YC\colon \lambda_{x,y,z}\le \adjfunc{\tau_{x,z}}(\varphi)} \,\vert\, X_{\tcount+1}=x, Z_{\tcount+1}=z}
      \\
      = \prob^{\mathcal{T}}\pr{\hat{Y}_{\tcount+1}\in \confReg_{z}(x; \adjfunc{\tau_{x,z}}(\varphi)) \,\vert\, X_{\tcount+1}=x, Z_{\tcount+1}=z}
      \\
      = \distr_{Y \mid X=x}(\confReg_{z}(x; \adjfunc{\tau_{x,z}}(\varphi)))
      \ge 1 - \alpha.
    \end{multline*}
    Therefore, we deduce that
    \begin{equation*}
      \prob^{\mathcal{T}}\pr{\lambda_{x,Y_{\tcount+1},z} \le \adjfunc{\tau_{x,z}}(\varphi) \,\vert\, X_{\tcount+1}=x, Z_{\tcount+1}=z}
      \ge 1 - \alpha - \tv(\textup{P}_{Y\mid X=x}; \distr_{Y \mid X=x}).
    \end{equation*}
    Combining the previous result with~\eqref{eq:prob-diff} and~\eqref{eq:prob-diff:2} shows that
    \begin{multline*}
      \prob^{\mathcal{T}}\pr{\lambda_{x,Y_{\tcount+1},z} \le \adjfunc{\tau_{x,z}}\pr{\q{1-\alpha}(\measure_{\tcount})} \,\vert\, X_{\tcount+1}=x, Z_{\tcount+1}=z}
      \ge 1 - \alpha - \tv(\textup{P}_{Y\mid X=x}; \distr_{Y \mid X=x})
      \\
      % + \prob^{\mathcal{T}}\pr{\varphi < \adjfunc{\tau_{x,z}}^{-1}\pr{\lambda_{x,Y_{\tcount+1},z}} \le \q{1-\alpha}(\measure_{\tcount}) \,\vert\, X_{\tcount+1}=x, Z_{\tcount+1}=z}
      - \prob^{\mathcal{T}}\pr{\q{1-\alpha}(\measure_{\tcount}) < \adjfunc{\tau_{x,z}}^{-1}\pr{\lambda_{x,Y_{\tcount+1},z}} \le \varphi \,\vert\, X_{\tcount+1}=x, Z_{\tcount+1}=z}.
    \end{multline*}
    Finally, using~\eqref{eq:eq:cov-cond:1} gives a lower bound on $\prob^{\mathcal{T}}\pr{Y_{\tcount+1} \in \mathcal{C}_{\alpha}(X_{\tcount+1}) \mid X_{\tcount+1}=x, Z_{\tcount+1}=z}$.

    \item By definition of the total variation distance, we have
    \begin{multline*}
      \prob^{\mathcal{T}}\pr{\lambda_{x,Y_{\tcount+1},z} \le \adjfunc{\tau_{x,z}}(\varphi) \,\vert\, X_{\tcount+1}=x, Z_{\tcount+1}=z}
      \\
      \le \prob^{\mathcal{T}}\pr{\lambda_{x,\hat{Y}_{\tcount+1},z}\le \adjfunc{\tau_{x,z}}(\varphi) \,\vert\, X_{\tcount+1}=x, Z_{\tcount+1}=z}
      + \tv(\textup{P}_{Y\mid X=x}; \distr_{Y \mid X=x}).
    \end{multline*}
    Moreover, \Cref{lem:tau} implies that
    \begin{equation*}
      \prob^{\mathcal{T}}\pr{\lambda_{x,\hat{Y}_{\tcount+1},z}\le \adjfunc{\tau_{x,z}}(\varphi) \,\vert\, X_{\tcount+1}=x, Z_{\tcount+1}=z}
      = \distr_{Y \mid X=x}(\confReg_{z}(x; \adjfunc{\tau_{x,z}}(\varphi))).
    \end{equation*}
    Therefore, we deduce that
    \begin{multline*}
      \prob^{\mathcal{T}}\pr{\lambda_{x,Y_{\tcount+1},z} \le \adjfunc{\tau_{x,z}}(\varphi) \,\vert\, X_{\tcount+1}=x, Z_{\tcount+1}=z}
      \\
      \le \distr_{Y \mid X=x}(\confReg_{z}(x; \adjfunc{\tau_{x,z}}(\varphi))) + \tv(\textup{P}_{Y\mid X=x}; \distr_{Y \mid X=x}).
    \end{multline*}
    Finally, combining the previous result with~\eqref{eq:prob-diff} and~\eqref{eq:prob-diff:2} shows that
    \begin{multline*}
      \prob^{\mathcal{T}}\pr{\lambda_{x,Y_{\tcount+1},z} \le \adjfunc{\tau_{x,z}}\pr{\q{1-\alpha}(\measure_{\tcount})} \,\vert\, X_{\tcount+1}=x, Z_{\tcount+1}=z}
      \\
      \le \distr_{Y \mid X=x}(\confReg_{z}(x; \adjfunc{\tau_{x,z}}(\varphi)))
      + \tv(\textup{P}_{Y\mid X=x}; \distr_{Y \mid X=x})
      + q_{\tcount+1}^{(x,z)}.
    \end{multline*}
  \end{enumerate}
\end{proof}

\subsection{Bound on $p_{\tcount+1}^{(x,z)}$ and $q_{\tcount+1}^{(x,z)}$}
\label{subsect:proof:cond-coverage}
% !TEX root = ../main.tex

The objective of this section is to study the conditional guarantee obtained in \Cref{supp:thm:coverage:conditional}. Under some assumptions, we have demonstrated that the conditional coverage is controlled as follows:
\begin{multline*}
  1 - \alpha - \tv(\textup{P}_{Y\mid X=x}; \distr_{Y \mid X=x}) - p_{\tcount+1}^{(x,z)}
  \le \prob^{\mathcal{T}}\pr{Y_{\tcount+1} \in \mathcal{C}_{\alpha}(X_{\tcount+1}) \mid X_{\tcount+1}=x, Z_{\tcount+1}=z}
  \\
  \le \distr_{Y \mid X=x}(\confReg_{z}(x; \adjfunc{\tau_{x,z}}(\varphi))) + \tv(\textup{P}_{Y\mid X=x}; \distr_{Y \mid X=x}) + q_{\tcount+1}^{(x,z)},
\end{multline*}
In the following, we consider the cumulative density functions $F\colon t\mapsto \prob^{\mathcal{T}}(\adjfunc{\tau_{X,Z}}^{-1}(\lambda_{X,Y,Z}) \le t)$ and $\hat{F}\colon t\mapsto \prob^{\mathcal{T}}(\adjfunc{\tau_{X,Z}}^{-1}(\lambda_{X,\hat{Y},Z}) \le t)$, where $(X,Y,Z)\sim \textup{P}_X\times \textup{P}_{Y\mid X}\times \PZ$ and $(X,\hat{Y},Z)\sim \textup{P}_X\times \distr_{Y \mid X}\times \PZ$.
We denote by $\mu$ and $\hat{\mu}$ the law of the random variables $\adjfunc{\tau_{X,Z}}^{-1}(\lambda_{X,Y,Z})$ and $\adjfunc{\tau_{X,Z}}^{-1}(\lambda_{X,\hat{Y},Z})$.
Moreover, recall that $\measure_{\tcount}=\frac{1}{\tcount+1}\sum_{k=1}^{\tcount}\delta_{\adjfunc{\bar{\tau}_{X_k}}^{-1}(\bar{\lambda}_k)}+\frac{1}{\tcount+1}\delta_{\infty}$.
% While $\E[p_{\tcount+1}^{(X_{\tcount+1},Z_{\tcount+1})}]\le \alpha$, studying $p_{\tcount+1}^{(x,z)}$ is challenging. However, we control this term in~\Cref{corr:suppl:cond-coverage}. 
Note, the quantile $\q{1-\alpha}(\measure_{\tcount})$ is an order statistic with a known distribution that converges to the true quantile $\q{1-\alpha}(\mu)$. The quantile is defined for any $t\in (0,1)$ by
\begin{equation}\label{eq:def:qepsilon}
  \q{t}(\nu) = \inf\{u\in\R\colon \nu((-\infty,u]) \ge t\},\qquad \text{where $\nu\in\{\mu,\mu_{\tcount},\hat{\mu}$\}}.
\end{equation}
\thmspace

\begin{theorem}\label{corr:suppl:cond-coverage}
  Assume \Cref{ass:confReg}-\Cref{ass:tau} hold, and let $x\in\R^d$ be such that $\distr_{Y \mid X=x}$ is a probability measure.
  For any $\epsilon\in [0,1-\alpha)$, if $p_{\epsilon} = \prob^{\mathcal{T}}(\adjfunc{\tau_{X,Z}}^{-1}(\lambda_{X,Y,Z})<\q{1-\alpha-\epsilon}(\mu))\le 1-\alpha$, then it follows that
  \begin{multline*}
    % \prob^{\mathcal{T}}\pr{Y_{\tcount+1} \in \mathcal{C}_{\alpha}(X_{\tcount+1}) \mid X_{\tcount+1}=x, Z_{\tcount+1}=z}
    % \ge 1 - \alpha - \tv(\textup{P}_{Y\mid X=x}; \distr_{Y \mid X=x})
    p_{\tcount+1}^{(x,z)}
    \le
    \prob^{\mathcal{T}}\pr{\q{1-\alpha-\epsilon}(\mu) < \adjfunc{\tau_{x,y}}^{-1}(\lambda_{x,Y_{\tcount+1},z}) \le \q{1-\alpha}(\hat{\mu}) \,\vert\, X_{\tcount+1}=x, Z_{\tcount+1}=z}
    \\
    + \exp\pr{-\tcount p_{\epsilon} (1-p_{\epsilon}) h\pr{\frac{1-\alpha-p_{\epsilon}}{p_{\epsilon}(1-p_{\epsilon})}}}
    % - \abs{F_x(\q{1-\alpha}(\hat{\mu})) - F_x(\q{1-\alpha-\epsilon}(\mu))}
    ,
  \end{multline*}
  where $h:u\mapsto (1+u) \log(1+u)-u$.
\end{theorem}

\begin{proof}
  Let $\epsilon\in[0,1-\alpha)$, $x\in\XC$, and consider
  \begin{align*}
    &A = \ac{\q{1-\alpha}(\measure_{\tcount}) < \q{1-\alpha-\epsilon}(\mu)},
    \\
    &B_{x,z} = \ac{y\in\YC\colon \adjfunc{\tau_{x,z}}(\q{1-\alpha-\epsilon}(\mu)) < \lambda_{x,y,z} \le \adjfunc{\tau_{x,z}}(\varphi)}.
  \end{align*}
  We have
  \begin{multline*}
    \prob^{\mathcal{T}}\pr{\adjfunc{\tau_{x,z}}\pr{\q{1-\alpha}(\measure_{\tcount})} < \lambda_{x,Y_{\tcount+1},z} \le \adjfunc{\tau_{x,z}}(\varphi) \,\vert\, X_{\tcount+1}=x, Z_{\tcount+1}=z}
    \\
    \le \prob^{\mathcal{T}}\pr{A \mid X_{\tcount+1}=x, Z_{\tcount+1}=z}
    + \prob^{\mathcal{T}}\pr{Y_{\tcount+1}\in B_{x,z} \mid X_{\tcount+1}=x, Z_{\tcount+1}=z}.
  \end{multline*}
  Now, let's upper bound the first term of the right-hand side equation.
  First, remark that
  \begin{equation*}
    \ac{\q{1-\alpha}(\measure_{\tcount}) < \q{1-\alpha-\epsilon}(\mu)}
    \Leftrightarrow \ac{\frac{1}{\tcount+1} \sum_{k=1}^{\tcount} \1_{\adjfunc{\bar{\tau}_k}^{-1}(\bar{\lambda}_k) < \q{1-\alpha-\epsilon}(\mu)} \ge 1 - \alpha}.
  \end{equation*}
  Thus, we deduce that 
  \begin{equation*}
    \prob^{\mathcal{T}}\pr{A \mid X_{\tcount+1}=x, Z_{\tcount+1}=z}
    % \\
    % = \prob^{\mathcal{T}}\pr{\lambda_{x,Y_{\tcount+1},z} \le \adjfunc{\tau_{x,z}}(\q{1-\alpha}(\measure_{\tcount})); \lambda_{x,Y_{\tcount+1},z}\wedge \tau_{x,z} \le \adjfunc{\tau_{x,z}}(\q{1-\alpha-\epsilon}(\mu)) \mid X_{\tcount+1}=x, Z_{\tcount+1}=z}
    % \\
    \le \prob^{\mathcal{T}}\pr{\sum_{k=1}^{\tcount} \1_{\adjfunc{\bar{\tau}_k}^{-1}(\bar{\lambda}_k) < \q{1-\alpha-\epsilon}(\mu)} \ge (\tcount+1) (1 - \alpha)}.
  \end{equation*}
  Recall that $p_{\epsilon} = \prob^{\mathcal{T}}(\adjfunc{\tau_{X,Z}}^{-1}(\lambda_{X,Y,Z})<\q{1-\alpha-\epsilon}(\mu))$, and also that we assume $p_{\epsilon}\le 1-\alpha$.
  Therefore, the Bennett's inequality \cite[Theorem 2]{boucheron2003concentration} implies that
  \begin{equation}\label{eq:prob:YinC:Ax}
    \prob^{\mathcal{T}}\pr{A \mid X_{\tcount+1}=x, Z_{\tcount+1}=z}
    % \\
    \le \exp\pr{-\tcount p_{\epsilon} (1-p_{\epsilon}) \, h\pr{\frac{(\tcount+1) (1 - \alpha) - \tcount p_{\epsilon}}{\tcount p_{\epsilon} (1-p_{\epsilon})}}},
  \end{equation}
  where $h:u\mapsto (1+u)\log(1+u)-u$.
  Moreover, define
  \begin{align*}
    &u_{\epsilon} = \frac{1-\alpha-p_{\epsilon}}{p_{\epsilon}(1-p_{\epsilon})},
    &\tilde{u}_{\epsilon} = \frac{(\tcount+1) (1 - \alpha) - \tcount p_{\epsilon}}{\tcount p_{\epsilon} (1-p_{\epsilon})}.
  \end{align*}
  We have $\tilde{u}_{\epsilon} \le u_{\epsilon}$, from the increasing property of $h$ it follows that
  \begin{equation*}
    \prob^{\mathcal{T}}\pr{A \mid X_{\tcount+1}=x, Z_{\tcount+1}=z}
    \le \exp\pr{-\tcount p_{\epsilon} (1-p_{\epsilon}) h(u_{\epsilon})}.
  \end{equation*}
  Furthermore, the definition of $B_{x,z}$ gives
  \begin{multline*}
    \prob^{\mathcal{T}}\pr{Y_{\tcount+1}\in B_{x,z} \mid X_{\tcount+1}=x, Z_{\tcount+1}=z}
    \\
    = \prob^{\mathcal{T}}\pr{\adjfunc{\tau_{x,z}}(\q{1-\alpha-\epsilon}(\mu)) < \lambda_{x,Y_{\tcount+1},z} \le \adjfunc{\tau_{x,z}}(\varphi) \mid X_{\tcount+1}=x, Z_{\tcount+1}=z}.
  \end{multline*}
  % By definition of $\q{1-\alpha-\epsilon}(\mu)$ provided in~\eqref{eq:def:qepsilon}, we have $\q{1-\alpha-\epsilon}(\mu)=\q{1-\alpha-\epsilon}(\mu)$.
  Moreover, for any $t\in (-\infty, \varphi)$, we have
  \begin{align*}
    \hat{F}(t)
    &= \prob^{\mathcal{T}}\pr{\adjfunc{\tau_{X,Z}}^{-1}(\lambda_{X,\hat{Y},Z})\le t}
    \\
    &= \int \prob^{\mathcal{T}}\pr{\adjfunc{\tau_{X,Z}}^{-1}(\lambda_{X,\hat{Y},Z})\le t \,\Big\vert\, X=x, Z=z} \bar{\distr}_{Z\mid X=x}(\rmd z) \, \textup{P}_X(\rmd x)
    \\
    &= \int \prob^{\mathcal{T}}\pr{\hat{Y} \in \confReg\pr{x, \adjfunc{\tau_{z,z}}(t)} \,\Big\vert\, X=x, Z=z} \bar{\distr}_{Z\mid X=x}(\rmd z) \, \textup{P}_X(\rmd x).
  \end{align*}
  Using~\Cref{ass:tau}, the bijective property of $\tau\mapsto \adjfunc{\tau}(\varphi)$ implies the existence of $\nu\in\mathsf{T}$, such that $\adjfunc{\nu}(\varphi)=\adjfunc{\tau_{z,z}}(t)$. Note that, $\nu<\tau_{x,z}$ otherwise it would lead to $\adjfunc{\nu}(\varphi)\ge \adjfunc{\tau_{x,z}}(\varphi) > \adjfunc{\tau_{x,z}}(t)$. The definition of $\tau_{x,z}$ shows that
  \[
    \prob^{\mathcal{T}}\pr{\hat{Y} \in \confReg\pr{x, \adjfunc{\nu}(\varphi)} \,\Big\vert\, X=x, Z=z}
    < 1-\alpha.
  \]
  Therefore, we deduce that $\q{1-\alpha}(\hat{\mu})\ge \varphi$, and we can conclude that
  \begin{multline}\label{eq:prob:YinC:Bx}
    \prob^{\mathcal{T}}\pr{Y_{\tcount+1}\in B_{x,z} \mid X_{\tcount+1}=x, Z_{\tcount+1}=z}
    \\
    \le \prob^{\mathcal{T}}\pr{\q{1-\alpha-\epsilon}(\mu) < \adjfunc{\tau_{x,y}}^{-1}(\lambda_{x,Y_{\tcount+1},z}) \le \q{1-\alpha}(\hat{\mu}) \mid X_{\tcount+1}=x, Z_{\tcount+1}=z}.
  \end{multline}
  Finally, combining~\eqref{eq:prob:YinC:Ax} and~\eqref{eq:prob:YinC:Bx} concludes the proof.
\end{proof}

Given $\alpha\in(0,1)$, define the threshold
\begin{equation}\label{eq:def:epsilon-n}
  \epsilon_{\tcount} 
  = \sqrt{\frac{8\alpha(1-\alpha) \log \tcount}{\tcount}}.
\end{equation}
\thmspace

\begin{lemma}\label{lem:suppl:bound:cond-coverage}
  If the distribution of $\adjfunc{\tau_{X,Z}}^{-1}(\lambda_{X,Y,Z})$ is continuous, then for all $\epsilon\in[0,1-\alpha)$, we have $p_{\epsilon} = \prob^{\mathcal{T}}(\adjfunc{\tau_{X,Z}}^{-1}(\lambda_{X,Y,Z})<\q{1-\alpha-\epsilon}(\mu))=1-\alpha-\epsilon$.
  Moreover, if $\epsilon_{\tcount}\le \frac{\alpha (1-\alpha)}{8}$, then it follows  
  \begin{equation*}
    \exp\pr{-\tcount p_{\epsilon_{\tcount}} (1-p_{\epsilon_{\tcount}}) h\pr{\frac{1-\alpha-p_{\epsilon_{\tcount}}}{p_{\epsilon_{\tcount}}(1-p_{\epsilon_{\tcount}})}}}
    \le \frac{1}{\tcount}
    ,
  \end{equation*}
  where $h:u\mapsto (1+u) \log(1+u)-u$.
\end{lemma}

\begin{proof}
  First, recall that $\q{1-\alpha-\epsilon}(\mu)$ is defined in~\eqref{eq:def:qepsilon}.
  If the distribution of $\adjfunc{\tau_{X,Z}}^{-1}(\lambda_{X,Y,Z})$ is continuous, then we have 
  \begin{multline*}
    1 - \alpha - \epsilon
    \le F(\q{1-\alpha-\epsilon}(\mu))
    = \sup_{\delta>0} F(\q{1-\alpha-\epsilon}(\mu)-\delta)
    \\
    \le \prob^{\mathcal{T}}\pr{\adjfunc{\tau_{X,Z}}^{-1}(\lambda_{X,Y,Z}) < \q{1-\alpha-\epsilon}(\mu)}
    = p_{\epsilon}
    \le 1 - \alpha - \epsilon.
  \end{multline*}
  Therefore, we deduce that $p_{\epsilon}=1-\alpha-\epsilon$.
  Let's denote
  \begin{align*}
    \delta_{\tcount} 
    &= (\tcount+1)(1-\alpha)-\tcount p_{\epsilon_{\tcount}},
    &
    u_{\tcount}
    &= \frac{(\tcount+1) (1 - \alpha) - \tcount p_{\epsilon_{\tcount}}}{\tcount p_{\epsilon_{\tcount}} (1-p_{\epsilon_{\tcount}})}.
  \end{align*}
  For any $u\ge 0$, remark that $\log(1+u) \ge u - u^2/2$. Thus, we deduce
  \begin{align}
    \nonumber
    \tcount p_{\epsilon_{\tcount}} (1-p_{\epsilon_{\tcount}}) \, h\pr{u_{\tcount}}
    &\ge \delta_{\tcount} \frac{(1+u_{\tcount}) \log(1+u_{\tcount}) - u_{\tcount}}{u_{\tcount}}
    \\
    \label{eq:bound:residu}
    &\ge \delta_{\tcount} \frac{u_{\tcount} (1 - u_{\tcount})}{2}.
  \end{align}
  Now, let's show that $u_{\tcount}\le 1/4$. We have
  \begin{align*}
    u_{\tcount}
    &= \frac{(\tcount+1) (1 - \alpha) - \tcount p_{\epsilon_{\tcount}}}{\tcount p_{\epsilon_{\tcount}} (1-p_{\epsilon_{\tcount}})}
    \\
    &= \frac{1-\alpha}{\tcount p_{\epsilon_{\tcount}} (1-p_{\epsilon_{\tcount}})} + \frac{1 - \alpha - p_{\epsilon_{\tcount}}}{p_{\epsilon_{\tcount}} (1-p_{\epsilon_{\tcount}})}
    \\
    &= \frac{1-\alpha}{\tcount (\alpha + \epsilon_{\tcount}) (1 - \alpha - \epsilon_{\tcount})} + \frac{\epsilon_{\tcount}}{(\alpha + \epsilon_{\tcount}) (1 - \alpha - \epsilon_{\tcount})}.
  \end{align*}
  Therefore, $u_{\tcount}\le 1/4$ if and only if
  \begin{equation*}
    \frac{1-\alpha}{\tcount} + \epsilon_{\tcount}
    \le \frac{(\alpha + \epsilon_{\tcount}) (1 - \alpha - \epsilon_{\tcount})}{4}.
  \end{equation*}
  The function $\epsilon\in[0,1/2-\alpha]\mapsto (\alpha + \epsilon) (1 - \alpha - \epsilon)$ is increasing. Since $\epsilon_{\tcount} \le \alpha(1-\alpha)/8 \le 1/2 - \alpha$, it is sufficient to prove that
  \begin{equation*}
    \frac{1-\alpha}{\tcount} + \epsilon_{\tcount}
    \le \frac{\alpha (1 - \alpha)}{4}.
  \end{equation*}
  Since $\epsilon_{\tcount}\le \alpha (1-\alpha) / 8$, we just need to show that
  \begin{equation}\label{eq:bound:to-check}
    \frac{1-\alpha}{\tcount}
    \le \frac{\alpha (1 - \alpha)}{8}, 
    \qquad \text{i.e.,} \qquad
    \frac{8 \alpha (1-\alpha)}{\tcount}
    \le \alpha^2 (1 - \alpha)
    .
  \end{equation}
  Again, using the fact that $\epsilon_{\tcount}\le \alpha (1-\alpha) / 8$, we deduce that
  \begin{equation*}
    \frac{8 \alpha (1-\alpha)}{\tcount}
    = \frac{\epsilon_{\tcount}^2}{\log \tcount}
    \le \frac{\alpha^2 (1-\alpha)^2}{8 \log \tcount}
    = \alpha^2 (1-\alpha) \times \frac{(1-\alpha)}{8 \log \tcount}.
  \end{equation*}
  Since $\frac{(1-\alpha)}{8 \log \tcount}\le 1$, we deduce that \eqref{eq:bound:to-check} holds. This concludes that $u_{\tcount}\le 1/4$.
  Moreover, for any $u\in[0,0.25]$, we have
  \begin{equation*}
    \delta_{\tcount} \frac{u (1 - u)}{2}
    \ge \frac{u\delta_{\tcount}}{4}.
  \end{equation*}
  Plugging the previous line in~\eqref{eq:bound:residu} implies that
  \begin{align}
    \nonumber
    \exp\pr{-\tcount p_{\epsilon_{\tcount}} (1-p_{\epsilon_{\tcount}}) \, h\pr{u_{\tcount}}}
    &\le \exp\pr{-\frac{\br{(\tcount+1) (1-\alpha) - \tcount p_{\epsilon_{\tcount}}}^2}{4\tcount p_{\epsilon_{\tcount}} (1-p_{\epsilon_{\tcount}})}}
    \\
    \nonumber
    &\le \exp\pr{-\frac{\pr{1 - \alpha + \tcount \epsilon_{\tcount}}^2}{4\tcount (\alpha + \epsilon_{\tcount}) (1-\alpha - \epsilon_{\tcount})}}
    \\
    \label{eq:bound:exp-tcount}
    &\le \exp\pr{-\frac{\tcount \epsilon_{\tcount}^2}{4(\alpha+\epsilon_{\tcount})(1 - \alpha - \epsilon_{\tcount})}}
    .
  \end{align}
  Lastly, since $\epsilon_{\tcount} \le \alpha$, it follows that
  \begin{align*}
    \frac{\tcount \epsilon_{\tcount}^2}{4(\alpha+\epsilon_{\tcount})(1 - \alpha - \epsilon_{\tcount})}
    = \frac{2 \alpha (1-\alpha) \log\tcount}{(\alpha+\epsilon_{\tcount})(1 - \alpha - \epsilon_{\tcount})}
    \ge \log\tcount.
  \end{align*}
  Combining the previous line with~\eqref{eq:bound:exp-tcount} completes the proof.  
\end{proof}

For any $\epsilon\in[0,\alpha)$, define
\begin{equation*}
  q_{\epsilon} 
  = \prob^{\mathcal{T}}(\adjfunc{\tau_{X,Z}}^{-1}(\lambda_{X,Y,Z})<\q{1-\alpha+\epsilon}(\mu)).
\end{equation*}

\begin{theorem}\label{corr:suppl:cond-coverage:qn}
  Assume \Cref{ass:confReg}-\Cref{ass:tau} hold, and let $x\in\R^d$ be such that $\distr_{Y \mid X=x}$ is a probability measure.
  If the distribution of $\adjfunc{\tau_{X,Z}}^{-1}(\lambda_{X,Y,Z})$ is continuous and $\tcount^{-1}\log\tcount\le 8^{-3}\alpha(1-\alpha)$, then, it holds
  \begin{multline}
    q_{\tcount+1}^{(x,z)}
    \le
    \frac{1}{\tcount}
    + \prob^{\mathcal{T}}\pr{\q{1-\alpha}(\hat{\mu}) < \adjfunc{\tau_{x,y}}^{-1}(\lambda_{x,Y_{\tcount+1},z}) \le \q{1-\alpha+\epsilon_{\tcount}}(\mu) \,\vert\, X_{\tcount+1}=x, Z_{\tcount+1}=z}
    ,
  \end{multline}
  where $\epsilon_{\tcount}$ is defined in~\eqref{eq:def:epsilon-n}.
\end{theorem}

\begin{proof}
  Let's consider
  \begin{align*}
    &A = \ac{\q{1-\alpha+\epsilon_{\tcount}}(\measure) < \q{1-\alpha}(\mu_{\tcount})},
    \\
    &B_{x,z} = \ac{y\in\YC\colon \adjfunc{\tau_{x,z}}(\q{1-\alpha}(\hat{\mu})) < \lambda_{x,y,z} \le \adjfunc{\tau_{x,z}}(\q{1-\alpha+\epsilon_{\tcount}}(\mu))}.
  \end{align*}
  We have
  \begin{multline}\label{eq:bound:cond-cov-qn:1}
    \prob^{\mathcal{T}}\pr{\adjfunc{\tau_{x,z}}\pr{\q{1-\alpha}(\hat{\measure})} < \lambda_{x,Y_{\tcount+1},z} \le \adjfunc{\tau_{x,z}}(\q{1-\alpha}(\mu_{\tcount})) \,\vert\, X_{\tcount+1}=x, Z_{\tcount+1}=z}
    \\
    \le \prob^{\mathcal{T}}\pr{A \mid X_{\tcount+1}=x, Z_{\tcount+1}=z}
    + \prob^{\mathcal{T}}\pr{Y_{\tcount+1}\in B_{x,z} \mid X_{\tcount+1}=x, Z_{\tcount+1}=z}.
  \end{multline}
  Now, let's upper bound the first term of the right-hand side equation.
  First, remark that
  \begin{equation*}
    \ac{\q{1-\alpha+\epsilon_{\tcount}}(\measure) < \q{1-\alpha}(\mu_{\tcount})}
    \Leftrightarrow \ac{\frac{1}{\tcount+1} \sum_{k=1}^{\tcount} \1_{\adjfunc{\bar{\tau}_k}^{-1}(\bar{\lambda}_k) < \q{1-\alpha+\epsilon_{\tcount}}(\mu)} < 1 - \alpha}.
  \end{equation*}
  Thus, we deduce that 
  \begin{equation*}
    \prob^{\mathcal{T}}\pr{A \mid X_{\tcount+1}=x, Z_{\tcount+1}=z}
    \le \prob^{\mathcal{T}}\pr{\sum_{k=1}^{\tcount} \1_{\adjfunc{\bar{\tau}_k}^{-1}(\bar{\lambda}_k) < \q{1-\alpha+\epsilon_{\tcount}}(\mu_{\tcount})} < (\tcount+1) (1 - \alpha)}.
  \end{equation*}
  Recall that $q_{\epsilon_{\tcount}} = \prob^{\mathcal{T}}(\adjfunc{\tau_{X,Z}}^{-1}(\lambda_{X,Y,Z})<\q{1-\alpha+\epsilon_{\tcount}}(\mu))$, and also that $q_{\epsilon_{\tcount}}< 1$ since the distribution of $\adjfunc{\tau_{X,Z}}^{-1}(\lambda_{X,Y,Z})$ is continuous with $1-\alpha+\epsilon_{\tcount}<1$.
  Therefore, the Bennett's inequality \cite[Theorem 2]{boucheron2003concentration} implies that
  \begin{equation*}
    \prob^{\mathcal{T}}\pr{A \mid X_{\tcount+1}=x, Z_{\tcount+1}=z}
    % \\
    \le \exp\pr{-\tcount q_{\epsilon_{\tcount}} (1-q_{\epsilon_{\tcount}}) \, h\pr{\frac{(\tcount+1) (1 - \alpha) - \tcount q_{\epsilon_{\tcount}}}{\tcount q_{\epsilon_{\tcount}} (1-q_{\epsilon_{\tcount}})}}},
  \end{equation*}
  where $h:u\mapsto (1+u)\log(1+u)-u$.
  Moreover, define
  \begin{align*}
    &u_{\epsilon_{\tcount}} = \frac{1-\alpha-q_{\epsilon_{\tcount}}}{q_{\epsilon_{\tcount}}(1-q_{\epsilon_{\tcount}})},
    &\tilde{u}_{\epsilon_{\tcount}} = \frac{(\tcount+1) (1 - \alpha) - \tcount q_{\epsilon_{\tcount}}}{\tcount q_{\epsilon_{\tcount}} (1-q_{\epsilon_{\tcount}})}.
  \end{align*}
  We have $\tilde{u}_{\epsilon_{\tcount}} \le u_{\epsilon_{\tcount}}$, from the increasing property of $h$ combined with~\Cref{lem:suppl:bound:cond-coverage}, it follows that
  \begin{equation*}
    \prob^{\mathcal{T}}\pr{A \mid X_{\tcount+1}=x, Z_{\tcount+1}=z}
    \le \exp\pr{-\tcount q_{\epsilon_{\tcount}} (1-q_{\epsilon_{\tcount}}) h(u_{\epsilon_{\tcount}})}
    \le \tcount^{-1}.
  \end{equation*}  
  The previous inequality combined with~\eqref{eq:bound:cond-cov-qn:1} concludes the proof.
\end{proof}

\subsection{Proof of \Cref{corr:coverage-conditional:asymptotic}}
\label{subsect:proof:asymptotic-cond-coverage}
% !TEX root = ../main.tex

\begin{theorem}\label{thm:prob-cov-conf:1}
  Assume \Cref{ass:confReg}-\Cref{ass:tau} and suppose the distributions of $\adjfunc{\tau_{X,Z}}^{-1}(\lambda_{X,Y,Z})$ and $\adjfunc{\tau_{X,Z}}^{-1}(\lambda_{X,\hat{Y},Z})$ are continuous.
  For any $\alpha\in(0,1)$ and $\rho>0$, it holds
  \begin{multline*}
      \prob^{\mathcal{T}}\pr{\abs{ \prob^{\mathcal{T}}\pr{Y_{\tcount+1}\in\mathcal{C}_{\alpha}(X_{\tcount+1}) \,\vert\, X_{\tcount+1},Z_{\tcount+1}}
      - 1 + \alpha} > \rho}
      \\
      \le \frac{ 2\tcount^{-1} + \sqrt{128 \alpha(1-\alpha)\tcount^{-1} \log \tcount} + 4 \tv(\textup{P}_{X,Y}; \textup{P}_{X}\times \distr_{Y \mid X}) }{ \rho }.
  \end{multline*}
\end{theorem}

\begin{proof}
    Let $\rho>0$ be fixed.
    Applying~\Cref{thm:coverage:conditional}, we obtain that
    \begin{multline}\label{eq:bound:cond-cov-asympt:1}
        1 - \alpha - \tv(\textup{P}_{Y\mid X=x}; \distr_{Y \mid X=x}) - p_{\tcount+1}^{(x,z)}
        \le \prob^{\mathcal{T}}\pr{Y_{\tcount+1} \in \mathcal{C}_{\alpha}(X_{\tcount+1}) \mid X_{\tcount+1}=x, Z_{\tcount+1}=z}
        \\
        \le \distr_{Y \mid X=x}(\confReg_{z}(x; \adjfunc{\tau_{x,z}}(\varphi))) + \tv(\textup{P}_{Y\mid X=x}; \distr_{Y \mid X=x}) + q_{\tcount+1}^{(x,z)}.
    \end{multline}

    \paragraph{Step 1: Lower bound.}

    Using the Markov's inequality implies that
    \begin{multline}\label{eq:bound:cond-cov-asympt:2}
        \prob^{\mathcal{T}}\pr{\prob^{\mathcal{T}}\pr{Y_{\tcount+1}\in\mathcal{C}_{\alpha}(X_{\tcount+1}) \,\vert\, X_{\tcount+1},Z_{\tcount+1}}
        < 1 - \alpha - \rho}
        \\
        \le \prob^{\mathcal{T}}\pr{
            \tv(\textup{P}_{Y\mid X}; \distr_{Y \mid X}) + p_{\tcount+1}^{(X,Z)}
            < \rho
        }
        \le \frac{ \E^{\mathcal{T}}\brbig{\tv(\textup{P}_{Y\mid X}; \distr_{Y \mid X})} + \E^{\mathcal{T}}\brbig{p_{\tcount+1}^{(X,Z)}} }{ \rho }.
    \end{multline}
    Moreover, using~\Cref{corr:suppl:cond-coverage} with $\Phi(\epsilon)=\epsilon[(u_\epsilon^{-1}-1) \log(1+u_{\epsilon}) - 1]$ and $u_{\epsilon}=\epsilon(\alpha+\epsilon)^{-1}(1-\alpha-\epsilon)$, it holds
    \begin{multline*}
        % \label{eq:bound:pXZ:1}
        \E^{\mathcal{T}}\brbig{p_{\tcount+1}^{(X,Z)}}
        = \prob^{\mathcal{T}}\pr{\q{1-\alpha}(\measure_{\tcount}) < \adjfunc{\tau_{X,Z}}^{-1}(\lambda_{X,Y,Z}) \le \varphi}
        \\
        % \nonumber
        \le \exp\pr{- \tcount \Phi(\epsilon)}
        + \prob^{\mathcal{T}}\pr{\q{1-\alpha-\epsilon}(\mu) < \adjfunc{\bar{\tau}_{\tcount+1}}^{-1}(\bar{\lambda}_{\tcount+1}) \le \q{1-\alpha}(\hat{\mu}) \,\big\vert\, X_{\tcount+1}=x, Z_{\tcount+1}=z}
        .
    \end{multline*}
    By~\Cref{lem:suppl:bound:cond-coverage}, if $\tcount^{-1}\log\tcount\le 8^{-3}\alpha(1-\alpha)$, then, setting $\epsilon_{\tcount} = \sqrt{8\alpha(1-\alpha)\tcount^{-1} \log \tcount}$ ensures that $\exp(- \tcount \Phi(\epsilon_{\tcount}))\le \tcount^{-1}$. We assume in the following that $\tcount^{-1}\log\tcount\le 8^{-3}\alpha(1-\alpha)$, because, if it not the case, the final upper bound obtained at the end of the proof is still valid. Thus, we get
    \begin{equation}\label{eq:bound:pXZ:2}
        \E^{\mathcal{T}}\brbig{p_{\tcount+1}^{(X,Z)}}
        \le \tcount^{-1}
        + \prob^{\mathcal{T}}\pr{\q{1-\alpha-\epsilon_{\tcount}}(\mu) < \adjfunc{\bar{\tau}_{\tcount+1}}^{-1}(\bar{\lambda}_{\tcount+1}) \le \q{1-\alpha}(\hat{\mu})}
        .
    \end{equation}
    Let's define $\bar{\gamma}$ by
    \begin{equation*}
        \bar{\gamma} 
        = \min(1, 1 - \alpha + \tv(\textup{P}_{X,Y}; \textup{P}_{X}\times\distr_{Y \mid X})).
    \end{equation*}
    We now show that $\q{1-\alpha}(\hat{\mu})\le \q{\bar{\gamma}}(\mu)$.
    By continuity of the cumulative density function of $\adjfunc{\tau_{X,Z}}^{-1}(\lambda_{X,\hat{Y},Z})$, we have
    \begin{align*}
        1 - \alpha
        &= \prob^{\mathcal{T}}\pr{\adjfunc{\tau_{X,Z}}^{-1}(\lambda_{X,\hat{Y},Z}) \le \q{1-\alpha}(\hat{\mu})}
        \\
        &\ge \prob^{\mathcal{T}}\pr{\adjfunc{\tau_{X,Z}}^{-1}(\lambda_{X,Y,Z}) \le \q{1-\alpha}(\hat{\mu})}
        - \tv(\textup{P}_{X,Y}; \textup{P}_{X}\times\distr_{Y \mid X}).
    \end{align*}
    Hence, it follows that
    \begin{align*}
        \prob^{\mathcal{T}}\pr{\adjfunc{\tau_{X,Z}}^{-1}(\lambda_{X,Y,Z}) \le \q{1-\alpha}(\hat{\mu})}
        \le \bar{\gamma}
        \le \prob^{\mathcal{T}}\pr{\adjfunc{\tau_{X,Z}}^{-1}(\lambda_{X,Y,Z}) \le \q{\bar{\gamma}}(\mu)}.
    \end{align*}
    Thus, the previous line implies that $\q{1-\alpha}(\hat{\mu}) \le \q{\bar{\gamma}}(\mu)$.
    Once again, using the continuity of the distribution of $\adjfunc{\tau_{X,Z}}^{-1}(\lambda_{X,Y,Z})$, we can write
    \begin{multline*}
        \prob^{\mathcal{T}}\pr{\q{1-\alpha}(\hat{\mu}) < \adjfunc{\tau_{x,y}}^{-1}(\lambda_{x,Y_{\tcount+1},z}) \le \q{1-\alpha+\epsilon_{\tcount}}(\mu)}
        \le \prob^{\mathcal{T}}\pr{\q{1-\alpha-\epsilon_{\tcount}}(\mu) < \adjfunc{\bar{\tau}_{\tcount+1}}^{-1}(\bar{\lambda}_{\tcount+1}) \le \q{\bar{\gamma}}(\mu)}
        \\
        = F\pr{ \q{\bar{\gamma}}(\mu) } - F\pr{ \q{1-\alpha-\epsilon_{\tcount}}(\mu) }
        = \bar{\gamma} + \epsilon_{\tcount} - 1 + \alpha
        \\
        = \tcount^{-1} + \sqrt{8\alpha(1-\alpha)\tcount^{-1} \log \tcount} + \tv(\textup{P}_{X,Y}; \textup{P}_{X}\times\distr_{Y \mid X}).
    \end{multline*}
    Plugging the previous inequality inside~\eqref{eq:bound:pXZ:2} yields
    \begin{equation*}
        \E^{\mathcal{T}}\brbig{p_{\tcount+1}^{(X,Z)}}
        \le \tcount^{-1} + \sqrt{8\alpha(1-\alpha)\tcount^{-1} \log \tcount} + \tv(\textup{P}_{X,Y}; \textup{P}_{X}\times\distr_{Y \mid X}).
    \end{equation*}
    Therefore, \eqref{eq:bound:cond-cov-asympt:2} implies that
    \begin{multline}\label{eq:bound:cond-cov-asympt:3}
        \prob^{\mathcal{T}}\pr{\prob^{\mathcal{T}}\pr{Y_{\tcount+1}\in\mathcal{C}_{\alpha}(X_{\tcount+1}) \,\vert\, X_{\tcount+1},Z_{\tcount+1}}
        < 1 - \alpha - \rho}
        \\
        \le \frac{ \tcount^{-1} + \sqrt{8\alpha(1-\alpha)\tcount^{-1} \log \tcount} + \tv(\textup{P}_{X,Y}; \textup{P}_{X}\times\distr_{Y \mid X}) + \E^{\mathcal{T}}\brbig{\tv(\textup{P}_{Y\mid X}; \distr_{Y \mid X})} }{ \rho }.
    \end{multline}

    \paragraph{Step 2: Upper bound.}
    
    Using~\eqref{eq:bound:cond-cov-asympt:1}, we obtain
    \begin{multline*}
        \prob^{\mathcal{T}}\pr{\prob^{\mathcal{T}}\pr{Y_{\tcount+1}\in\mathcal{C}_{\alpha}(X_{\tcount+1}) \,\vert\, X_{\tcount+1},Z_{\tcount+1}}
        > 1 - \alpha + \rho}
        \\
        \le \prob^{\mathcal{T}}\pr{
            \distr_{Y \mid X=X}(\confReg_{Z}(X; \adjfunc{\tau_{X,Z}}(\varphi))) + \tv(\textup{P}_{Y\mid X=X}; \distr_{Y \mid X=X}) + q_{\tcount+1}^{(X,Z)}
            > 1 - \alpha + \rho
        }.
    \end{multline*}
    The continuity of the distribution of $\adjfunc{\tau_{X,Z}}^{-1}(\lambda_{X,\hat{Y},Z})$ implies
    \begin{equation*}
        1 - \alpha
        = \prob^{\mathcal{T}}\pr{ \adjfunc{\tau_{X,Z}}^{-1}(\lambda_{X,\hat{Y},Z}) \le \varphi}
        = \int \distr_{Y \mid X=x}(\confReg_{z}(x; \adjfunc{\tau_{x,z}}(\varphi))) \bar{\distr}_{Z \mid X=x}(\rmd z) \textup{P}_X(\rmd x).
    \end{equation*}
    Since $\distr_{Y \mid X=x}(\confReg_{z}(x; \adjfunc{\tau_{x,z}}(\varphi)))\ge 1-\alpha$, we deduce that $\distr_{Y \mid X=x}(\confReg_{z}(x; \adjfunc{\tau_{x,z}}(\varphi)))=1-\alpha$ almost surely.
    Therefore, using the Markov's inequality gives
    \begin{multline}\label{eq:bound:qXZ:1}
        \prob^{\mathcal{T}}\pr{\prob^{\mathcal{T}}\pr{Y_{\tcount+1}\in\mathcal{C}_{\alpha}(X_{\tcount+1}) \,\vert\, X_{\tcount+1},Z_{\tcount+1}}
        > 1 - \alpha + \rho}
        \le \frac{ \E^{\mathcal{T}}\brbig{\tv(\textup{P}_{Y\mid X}; \distr_{Y \mid X})} + \E^{\mathcal{T}}\brbig{q_{\tcount+1}^{(X,Z)}} }{ \rho }.
    \end{multline}
    Moreover, applying \Cref{corr:suppl:cond-coverage:qn} shows that
    \begin{equation}\label{eq:bound:qXZ:2}
        q_{\tcount+1}^{(x,z)}
        \le \tcount^{-1}
        + \prob^{\mathcal{T}}\pr{\q{1-\alpha}(\hat{\mu}) < \adjfunc{\tau_{x,y}}^{-1}(\lambda_{x,Y_{\tcount+1},z}) \le \q{1-\alpha+\epsilon_{\tcount}}(\mu) \,\vert\, X_{\tcount+1}=x, Z_{\tcount+1}=z}
        .
    \end{equation}
    Let's define $\underline{\gamma}$ by
    \begin{equation*}
        \underline{\gamma} 
        = \min(1, 1 - \alpha - \tv(\textup{P}_{X,Y}; \textup{P}_{X}\times\distr_{Y \mid X})).
    \end{equation*}
    We now show that $\q{\underline{\gamma}}(\mu)\le \q{1-\alpha}(\hat{\mu})$.
    By continuity of the cumulative density function of $\adjfunc{\tau_{X,Z}}^{-1}(\lambda_{X,\hat{Y},Z})$, we have
    \begin{align*}
        1 - \alpha
        &= \prob^{\mathcal{T}}\pr{\adjfunc{\tau_{X,Z}}^{-1}(\lambda_{X,\hat{Y},Z}) \le \q{1-\alpha}(\hat{\mu})}
        \\
        &\le \prob^{\mathcal{T}}\pr{\adjfunc{\tau_{X,Z}}^{-1}(\lambda_{X,Y,Z}) \le \q{1-\alpha}(\hat{\mu})}
        + \tv(\textup{P}_{X,Y}; \textup{P}_{X}\times\distr_{Y \mid X}).
    \end{align*}
    Hence, it follows that
    \begin{align*}
        \underline{\gamma}
        \le \prob^{\mathcal{T}}\pr{\adjfunc{\tau_{X,Z}}^{-1}(\lambda_{X,Y,Z}) \le \q{1-\alpha}(\hat{\mu})}.
    \end{align*}
    Thus, we deduce that $\q{1-\alpha}(\hat{\mu}) \ge \q{\underline{\gamma}}(\mu)$.
    Using the continuity of the distribution of $\adjfunc{\tau_{X,Z}}^{-1}(\lambda_{X,Y,Z})$, we can write
    \begin{multline*}
        \prob^{\mathcal{T}}\pr{\q{1-\alpha}(\hat{\mu}) < \adjfunc{\tau_{x,y}}^{-1}(\lambda_{x,Y_{\tcount+1},z}) \le \q{1-\alpha+\epsilon_{\tcount}}(\mu)}
        \le \prob^{\mathcal{T}}\pr{\q{\underline{\gamma}}(\mu) < \adjfunc{\tau_{x,y}}^{-1}(\lambda_{x,Y_{\tcount+1},z}) \le \q{1-\alpha+\epsilon_{\tcount}}(\mu)}
        \\
        = F\pr{ \q{1-\alpha+\epsilon_{\tcount}}(\mu) } - F\pr{ \q{\underline{\gamma}}(\mu) }
        = \epsilon_{\tcount} - 1 + \alpha - \underline{\gamma}
        \\
        = \tcount^{-1} + \sqrt{8\alpha(1-\alpha)\tcount^{-1} \log \tcount} + \tv(\textup{P}_{X,Y}; \textup{P}_{X}\times\distr_{Y \mid X}).
    \end{multline*}
    Plugging the previous inequality inside~\eqref{eq:bound:qXZ:2} yields
    \begin{equation*}
        \E^{\mathcal{T}}\brbig{q_{\tcount+1}^{(X,Z)}}
        \le \tcount^{-1} + \sqrt{8\alpha(1-\alpha)\tcount^{-1} \log \tcount} + \tv(\textup{P}_{X,Y}; \textup{P}_{X}\times\distr_{Y \mid X}).
    \end{equation*}
    Therefore, \eqref{eq:bound:qXZ:1} implies that
    \begin{multline}\label{eq:bound:qXZ:3}
        \prob^{\mathcal{T}}\pr{\prob^{\mathcal{T}}\pr{Y_{\tcount+1}\in\mathcal{C}_{\alpha}(X_{\tcount+1}) \,\vert\, X_{\tcount+1},Z_{\tcount+1}}
        < 1 - \alpha - \rho}
        \\
        \le \frac{ \tcount^{-1} + \sqrt{8\alpha(1-\alpha)\tcount^{-1} \log \tcount} + \tv(\textup{P}_{X,Y}; \textup{P}_{X}\times\distr_{Y \mid X}) + \E^{\mathcal{T}}\brbig{\tv(\textup{P}_{Y\mid X}; \distr_{Y \mid X})} }{ \rho }.
    \end{multline}

    \paragraph{Step 3: Bound on $\E^{\mathcal{T}}\brbig{\tv(\textup{P}_{Y\mid X}; \distr_{Y \mid X})}$.}

    Let's denote $\nu_{Y\mid X=x} = 2^{-1} (\textup{P}_{Y\mid X=x} + \distr_{Y \mid X=x})$.
    Since $\textup{P}_{Y\mid X=x}\ll \nu_{Y\mid X=x}$ and $\distr_{Y\mid X=x}\ll \nu_{Y\mid X=x}$, there exists two Radon--Nikodym derivatives $g_1(x,\cdot)$ and $g_1(x,\cdot)$ of $\textup{P}_{Y\mid X=x}$ and $\distr_{Y\mid X=x}$ with respect to $\nu_{Y\mid X=x}$.
    Moreover, $g_1$ and $g_2$ are also the Radon--Nikodym derivatives of $\textup{P}_{X,Y}$ and $\textup{P}_X \times \distr_{Y\mid X}$ with respect to $\textup{P}_X \times \nu_{Y\mid X}$.
    By definition of the total variation distance, we have
    \begin{align}
        \nonumber
        \E^{\mathcal{T}}\brbig{\tv(\textup{P}_{Y\mid X}; \distr_{Y \mid X})}
        &= \int \tv(\textup{P}_{Y\mid X}; \distr_{Y \mid X}) \textup{P}_X(\rmd x)
        \\
        \nonumber
        &= \frac{1}{2} \int \abs{g_1(x,y) - g_2(x,y)} \nu_{Y\mid X=x} \textup{P}_X(\rmd x)
        \\
        \label{eq:bound:tv:1}
        &= \tv(\textup{P}_{X,Y}; \textup{P}_{X}\times\distr_{Y \mid X}).
    \end{align}

    \paragraph{Step 4: Combination.}

    Finally, using~\eqref{eq:bound:cond-cov-asympt:3}-\eqref{eq:bound:qXZ:3} and~\eqref{eq:bound:tv:1}, it follows that
    \begin{multline*}
        \prob^{\mathcal{T}}\pr{\abs{ \prob^{\mathcal{T}}\pr{Y_{\tcount+1}\in\mathcal{C}_{\alpha}(X_{\tcount+1}) \,\vert\, X_{\tcount+1},Z_{\tcount+1}}
        - 1 + \alpha} > \rho}
        \\
        \le \frac{ 2\tcount^{-1} + \sqrt{128 \alpha(1-\alpha)\tcount^{-1} \log \tcount} + 4 \tv(\textup{P}_{X,Y}; \textup{P}_{X}\times \distr_{Y \mid X}) }{ \rho }.
    \end{multline*}
    Note that the proof assumes $\tcount^{-1}\log\tcount \leq 8^{-3}\alpha(1-\alpha)$. To ensure the validity of the previous bound even when this assumption does not hold, we increased the term $\sqrt{32 \alpha(1-\alpha)\tcount^{-1} \log \tcount}$ to $\sqrt{128 \alpha(1-\alpha)\tcount^{-1} \log \tcount}$.
\end{proof}

\begin{theorem}
    Assume \Cref{ass:confReg}-\Cref{ass:tau}-\Cref{ass:tv:cv-prob} hold.
    If the distributions of $\adjfunc{\tau_{X,Z}}^{-1}(\lambda_{X,Y,Z})$ and $\adjfunc{\tau_{X,Z}}^{-1}(\lambda_{X,\hat{Y},Z})$ are continuous, then, $\forall \epsilon\in(0,1)$ there exists $(\Lambda_{\tcount}^{(\epsilon)})_{\tcount\in\N}$ such that $\liminf_{\tcount\to\infty}\prob((X_{\tcount+1},Z_{\tcount+1}) \in \Lambda_{\tcount}^{(\epsilon)}) \ge 1 - \epsilon$ and also
    \begin{equation*}
        \sup_{(x,z)\in\Lambda_{\tcount}^{(\epsilon)}} \abs{ \prob^{\mathcal{T}}\pr{Y_{\tcount+1}\in\mathcal{C}_{\alpha}(X_{\tcount+1}) \,\vert\,( X_{\tcount+1},Z_{\tcount+1})=(x,z)}
        - 1 + \alpha }
        = \Oh_{\prob}\pr{ \sqrt{\tcount^{-1}\log\tcount} + r_{\tcount}}
        .
    \end{equation*}
\end{theorem}

\begin{proof}
    First of all, define the following variables
    \begin{align*}
        &c_{\tcount+1}(x,z)
        = \abs{ \prob^{\mathcal{T}}\pr{Y_{\tcount+1}\in\mathcal{C}_{\alpha}(X_{\tcount+1}) \,\vert\,( X_{\tcount+1},Z_{\tcount+1})=(x,z)}
            - 1 + \alpha }
        ,
        \\
        &d_{\tcount}
        = \tv(\textup{P}_{X,Y}; \textup{P}_{X}\times \distr_{Y \mid X}^{(m_{\tcount})}).
    \end{align*}
    Applying~\Cref{thm:prob-cov-conf:1}, we obtain
    \begin{multline*}
        \prob\pr{
            c_{\tcount+1}(X_{\tcount+1},Z_{\tcount+1}) > \rho
        }
        \le \prob\pr{ d_{\tcount} > r_{\tcount} }
        + \prob\pr{
            c_{\tcount+1}(X_{\tcount+1},Z_{\tcount+1}) > \rho;
            d_{\tcount} \le r_{\tcount}
        }
        \\
        \le \prob\pr{ d_{\tcount} > r_{\tcount} }
        + \E\br{
            \1_{d_{\tcount} \le r_{\tcount}}
            \prob^{\mathcal{T}}\pr{
                c_{\tcount+1}(X_{\tcount+1},Z_{\tcount+1}) > \rho
            }
        }
        \\
        \le \prob\pr{ d_{\tcount} > r_{\tcount} }
        + \frac{ 2\tcount^{-1} + \sqrt{128 \alpha(1-\alpha)\tcount^{-1} \log \tcount} + 4 r_{\tcount} }{ \rho }.
    \end{multline*}
    Finally, using~\Cref{ass:tv:cv-prob},  we get $\lim_{\tcount\to\infty} \prob( d_{\tcount} > r_{\tcount} ) = 0$.
    Therefore, for any $\epsilon>0$, there exist $M_{\epsilon}>0$ and $\tilde{\tcount}_{\epsilon}\in\N$ such that, $\forall \tcount\ge \tilde{\tcount}_{\epsilon}$, it holds
    \begin{equation}\label{eq:bound:c-epsilon}
        \prob\pr{
            c_{\tcount+1}(X_{\tcount+1},Z_{\tcount+1}) > M_{\epsilon}\times\pr{ \sqrt{\tcount^{-1}\log\tcount} + r_{\tcount}}
        }
        \le \epsilon.
    \end{equation}
    Given $\epsilon\in(0,1)$, let's consider the following set
    \begin{equation*}
        \Lambda_{\tcount}^{(\epsilon)}
        = \ac{(X_{\tcount+1}(\omega), Z_{\tcount+1}(\omega)) \colon \omega\in\Omega, c_{\tcount+1}(X_{\tcount+1},Z_{\tcount+1})(\omega) \le M_{\epsilon} \times \pr{ \sqrt{\tcount^{-1}\log\tcount} + r_{\tcount}}}.
    \end{equation*}
    \Cref{eq:bound:c-epsilon} implies that
    \begin{equation*}
        \liminf_{\tcount\to\infty} \prob\pr{(X_{\tcount+1},Z_{\tcount+1}) \in \Lambda_{\tcount}^{(\epsilon)}}
        \ge 1 - \epsilon,
    \end{equation*}
    and by definition of $\Lambda_{\tcount}^{(\epsilon)}$, we also have
    \begin{equation*}
        \sup_{(x,z)\in\Lambda_{\tcount}^{(\epsilon)}} c_{\tcount+1}(x,z)
        = \Oh_{\prob}\pr{ \sqrt{\tcount^{-1}\log\tcount} + r_{\tcount}}
        .
    \end{equation*}
\end{proof}

Note that, \eqref{eq:bound:c-epsilon} also shows that
\begin{equation*}
    \abs{ \prob^{\mathcal{T}}\pr{Y_{\tcount+1}\in\mathcal{C}_{\alpha}(X_{\tcount+1}) \,\vert\, X_{\tcount+1}, Z_{\tcount+1}} - 1 + \alpha }
    = \Oh_{\prob}\pr{\tcount^{-1/2} \sqrt{\log \tcount}  + r_{\tcount}}
    .
  \end{equation*}

\subsection{Additional results}
\label{subsect:proof:quantile-rewritten}
% !TEX root = ../main.tex

Let's denote the conditional c.d.f of $\adjfunc{\tau_{X,Z}}^{-1}(\lambda_{X,Y,Z})$ by
\begin{equation*}
  F_{x,z}(\cdot)
  = \int_{\XC\times\ZC} \prob\pr{\adjfunc{\tau_{x,z}}^{-1}(\lambda_{x,Y,z}) \le \cdot \mid (X,Z)=(x,z)} \bar{\distr}_{Z\mid X=x}(\rmd z) \textup{P}_{X}(\rmd x).
\end{equation*}

\begin{lemma}
  Assume that $\q{1-\alpha}(\measure_{\tcount})\to \varphi$ almost-surely as $\tcount\to\infty$.
  If $F_{X,Z}$ is continuous almost-surely, then $\lim_{\tcount\to\infty} p_{\tcount+1}^{(x,z)}=0$, $\bar{\distr}_{Z\mid X} \times \textup{P}_{X}$-almost everywhere.
\end{lemma}

\begin{proof}
  First, define the following sets:
  \begin{align*}
    &A = \ac{\omega\in\Omega\colon \lim_{\tcount\to\infty} \q{1-\alpha}(\measure_{\tcount}(\omega)) = \varphi},
    \\
    &B = \ac{\omega\in\Omega\colon F_{X(\omega),Z(\omega)} \text{ is continuous}}.
  \end{align*}
  For all $\omega\in A\cap B$, it holds
  \begin{equation*}
    \lim_{\tcount\to\infty} F_{X(\omega),Z(\omega)}\pr{\q{1-\alpha}\pr{\measure_{\tcount}(\omega)} \wedge \varphi}
    = F_{X(\omega),Z(\omega)}\pr{\varphi}.
  \end{equation*}
  Moreover, note that we can write
  \begin{equation*}
    p_{\tcount+1}^{(x,z)}
    = F_{x,z}(\varphi) - F_{x,z}(\varphi\wedge\q{1-\alpha}(\measure_{\tcount})).
  \end{equation*}
  Hence, we deduce that
  \begin{align*}
    1
    = \prob\pr{A\cap B}
    &\le \prob\pr{
      \omega\in\Omega \colon
      \lim_{\tcount\to\infty} F_{X(\omega),Z(\omega)}\pr{\q{1-\alpha}\pr{\measure_{\tcount}(\omega)}\wedge \varphi}
      = F_{X(\omega),Z(\omega)}\pr{\varphi}
    }
    \\
    &= \prob\pr{
      \lim_{\tcount\to\infty} p_{\tcount+1}^{(X,Z)}
      = 0
    }
    \\
    &= \int_{\XC\times\ZC} \prob\pr{
      \lim_{\tcount\to\infty} p_{\tcount+1}^{(x,z)}
      = 0 \,\big\vert\, (X,Z)=(x,z)
    } \, \textup{P}_{Z\mid X=x}(\rmd z) \, \textup{P}_{X}(\rmd x).
  \end{align*}
  The last line implies that $p_{\tcount+1}^{(x,z)}\to 0$ almost $\textup{P}_{Z\mid X} \times \textup{P}_{X}$-everywhere.
\end{proof}

The prediction set, defined in~\eqref{eq:pred-set}, is derived from the $(1-\alpha)$-quantile of the conformity scores $\{\adjfunc{\bar{\tau}_{k}}^{-1}(\bar{\lambda}_{k})\}_{k=1}^{\tcount}\cup\{\infty\}$.
However, $\{\infty\}$ can be removed from these conformity scores.
Inspired by~\cite{romano2019conformalized,sesia2020comparison},  we prove a corollary of~\Cref{thm:coverage:marginal}. Its result demonstrates the marginal validity of the prediction set defined as
\begin{equation}\label{eq:def:C-bar-alpha}
  \bar{\mathcal{C}}_{\alpha}(x)
  = \confReg_{z}\pr{x; \adjfunc{\tau_{x,z}} \prBig{
    \q{(1-\alpha)(1+\tcount^{-1})} \prbig{\textstyle \frac{1}{\tcount} \sum_{k=1}^{\tcount} \delta_{\adjfunc{\bar{\tau}_{k}}^{-1}(\bar{\lambda}_{k})}}
  }}.
\end{equation}
While the prediction set $\bar{\mathcal{C}}_{\alpha}(x)$ relies on the quantile of the distribution $\frac{1}{\tcount} \sum_{k=1}^{\tcount} \delta_{\adjfunc{\bar{\tau}_{k}}^{-1}(\bar{\lambda}_{k})}$, its proof reveals that this prediction set is equivalent to $\mathcal{C}_{\alpha}(x)$.
\thmspace

\begin{corollary}\label{cor:coverage:marginal}
  Under the same assumptions as in \Cref{thm:coverage:marginal}, for any $\alpha\in [1/(\tcount+1), 1]$, we have
  \begin{equation*}
    1 - \alpha
    \le \prob\pr{Y_{\tcount+1} \in \bar{\mathcal{C}}_{\alpha}(X_{\tcount+1})}
    < 1 - \alpha + \frac{1}{\tcount+1},
  \end{equation*}
  where the upper bound only holds if the conformity scores $\{\adjfunc{\bar{\tau}_{k}}^{-1}(\bar{\lambda}_{k})\}_{k=1}^{\tcount+1}$ are almost surely distinct.
\end{corollary}

\begin{proof}
  Let $\alpha\in\R$ such that $(\tcount+1)^{-1} \le \alpha \le 1$, and recall that
  \begin{equation*}
    \measure_{\tcount} = \frac{1}{\tcount+1}\sum_{k=1}^{\tcount} \delta_{\adjfunc{\bar{\tau}_{k}}^{-1}(\bar{\lambda}_{k})} + \frac{1}{\tcount+1} \delta_{\infty}.
  \end{equation*}
  Since $\alpha\ge(\tcount+1)^{-1}$, the quantile $\q{1-\alpha}(\measure_{\tcount})$ is the $k_{\alpha}$th order statistic of $V_1,\ldots,V_{\tcount}$, where
  \begin{equation*}
    V_k = \adjfunc{\bar{\tau}_{k}}^{-1}(\bar{\lambda}_{k}),
    \quad \text{and} \quad
    k_\alpha = \lceil (1-\alpha) (\tcount+1) \rceil.
  \end{equation*}
  However, $\forall \beta\in (\frac{k_{\alpha}-1}{\tcount}, \frac{k_{\alpha}}{\tcount}]$, we have
  \begin{equation*}
    \q{\beta}\pr{\textstyle \frac{1}{\tcount} \sum_{k=1}^{\tcount} \delta_{V_k}} = V_{(k_{\alpha})}.
  \end{equation*}
  Since $\mathcal{C}_{\alpha}(X_{\tcount+1}) =  \confReg_{Z_{\tcount+1}}\prn{X_{\tcount+1}; \adjfunc{\bar{\tau}_{\tcount+1}}\prn{V_{(k_{\alpha})}}}$, \Cref{thm:coverage:marginal} implies that
  \begin{equation*}
    1 - \alpha
    \le 
    \prob\pr{Y_{\tcount+1} \in \confReg_{Z_{\tcount+1}}\pr{X_{\tcount+1}; \adjfunc{\bar{\tau}_{\tcount+1}}\pr{\q{\beta}\pr{\textstyle \frac{1}{\tcount} \sum_{k=1}^{\tcount} \delta_{V_k}}}}}
    < 1 - \alpha + \frac{1}{\tcount + 1}.
  \end{equation*}
  Setting $\beta = (1-\alpha)(1+\tcount^{-1})$ in the previous inequality and using the definition of $\bar{\mathcal{C}}_{\alpha}(X_{\tcount+1})$ given in~\eqref{eq:def:C-bar-alpha} concludes the proof.
\end{proof}

\section{Experimental Setup and Results}
\label{suppl:expriments}
% !TEX root = ../main.tex

  This section aims to provide a comprehensive understanding of the \algo\ algorithm. We want to further explore the \algo\ approach and to better explain the key concepts.

\paragraph{Choice of $\adjfunc{t}$.}
  We present examples of mappings $\adjfunc{t}$ and their inverses $\adjfunc{\tau}^{-1}$ in~\Cref{table:adjustment}. 
  The choice of the mapping $\adjfunc{t}$ is crucial for the performance of the method, and we investigate their impact in~\Cref{sec:expriments}.
  For instance, choosing $\adjfunc{\tau}(\lambda)=\tau\lambda$ results in approximately conditionally valid prediction sets, as long as $\distr_{Y \mid X=x}$ accurately estimates the conditional distribution $\textup{P}_{Y \mid X=x}$; see~\Cref{thm:coverage:conditional}-\Cref{corr:coverage-conditional:asymptotic}. Initially we also considered other adjustment functions based on exponent, sigmoid and $\tanh$ functions, but they all performed worse than linear and sum. As we show later in~\ref{table:combined_real}, these two selected adjustment function perform similarly, showing only marginal differences on some datasets. Designing new adjustment functions is a possible future research direction.

\subsection{Highest predictive density (HPD) regions}
\label{sec:hpd}
\paragraph{\algo\ with Explicit Conditional Density estimate: \hpdalgo.}
Assume that an estimator the conditional density function is known, denoted by $\pdfdistr_{Y\mid X=x}$.
The confidence set is defined as $\confReg(x; t)= \{ y \in \YC\colon \pdfdistr_{Y\mid X=x}(y) \ge -t\}$. We omit the variable $z$ from the notation, as we do not consider exogenous randomization in this case.
The parameter $\tau_x$ is obtained by solving 
\begin{equation}
\label{eq:tau-hpalgo}
\textstyle
 \tau_{x} = \argmin\left\{\tau\in\R\colon \int_{\confReg(x; \tau)} \pdfdistr_{Y \mid X=x}(y) \, \rmd y \ge 1 - \alpha\right\}.
\end{equation}
We then compute $\lambda_{x,y}=-\pdfdistr_{Y\mid X=x}(y)$ and derive the prediction set as
\[
  \textstyle 
  \mathcal{C}_\alpha(x)
  = \ac{ y \in \YC\colon \pdfdistr_{Y\mid X=x}(y) \ge -\adjfunc{\tau_x}\pr{\q{1-\alpha}\pr{\measure_{\tcount}}}}.
\]
If we take $\adjfunc{\tau}(\lambda)=\lambda$ and $\varphi=1$, the method shares similarity with the \texttt{CD-split} method, proposed in~\citet{izbicki2019flexible}. 
While \texttt{CD-split} uses $\lambda_{x,y}$ as the conformity score, our method uses $\adjfunc{\tau_x}^{-1}(\lambda_{x,y})$, which incorporates the information from $\tau_x$ to modify $\pdfdistr_{Y\mid X=x}(y)$. The \hpdalgo\ workflow is summarized in Algorithm~\ref{algo:CP2-HPD}.

Of course, the computation of~\eqref{eq:tau-hpalgo} is in general highly non-trivial. \cite{izbicki2019flexible} suggested to use binning, therefore approximating the conditional predictive distribution with histograms. The method is restricted to the case where the dimension of $\YC$ the response is small; see~\citet{izbicki2019flexible} for the case of $\YC= \R$. When the dimension becomes larger, then the estimation of HPD is typically based on Monte Carlo methods, thus requiring the introduction of auxiliary variables.

\begin{algorithm}[H]
  \caption{\hpdalgo}
  \label{algo:CP2-HPD}
  \begin{algorithmic}[]
      \State {\bfseries Input:} dataset $\{(X_k,Y_k)\}_{k\in[\tcount]}$, significance $\alpha$, conditional density $\pdfdistr_{Y \mid X}$, function $\adjfunc{t}$.
      \State \algorithmiccomment{Compute the $(1-\alpha)$-quantile}
      \For{$k=1$ {\bfseries to} $\tcount$}
          \State Set $\bar{\lambda}_{k} = -\pdfdistr_{Y \mid X=X_k}(Y_k)$
          \State Set $\bar{\tau}_k = \tau_{X_k}$ as given in~\eqref{eq:def:tau}
      \EndFor
      \State $\q{1-\alpha}\pr{\measure_{\tcount}} \gets \lceil (1-\alpha) (\tcount+1) \rceil$-th smallest value in $\{\adjfunc{\bar{\tau}_{k}}^{-1}(\bar{\lambda}_k)\}_{k\in[\tcount]}\cup\{\infty\}$
      \State \algorithmiccomment{Compute the prediction set for a new point $x\in\XC$}
      \State Compute $\tau_{x}$ in \eqref{eq:def:tau}.
      \State {\bfseries Output:} $\mathcal{C}_{\alpha}(x) = \{ y\in\YC \colon \pdfdistr_{Y \mid X=x}(y) \ge -\adjfunc{\tau_{x}}\prt{\q{1-\alpha}\prt{\measure_{\tcount}}} \}$.
  \end{algorithmic}
\end{algorithm}

%  \begin{table}[t]
%    \centering
%    \begin{tabular}{lccccc}
%      \toprule
%      \textsc{Adjustment} & Linear & Diff & Exp & Tanh & Sigmoid \\
%      \midrule
%      $\adjfunc{\tau}(\lambda)$ & $\tau \lambda$ & $\tau + \lambda$ & $\exp(\tau \lambda)$ & $\tan(\tau \lambda)$ & $(1 + \exp(-\lambda \tau))^{-1}$ \\
%      $\adjfunc{\tau}^{-1}(\lambda)$ & $\tau^{-1} \lambda$ & $\lambda - \tau$ & $\tau^{-1} \log \lambda$ & $ \tau^{-1} \mathrm{arctan}\,\lambda$ & $\tau^{-1} \log( (1 - \lambda)^{-1} \lambda)$ \\
%      \bottomrule
%    \end{tabular}
%    \vskip 0.08in
%    \caption{Adjustment Functions $\adjfunc{t}$ and their inverses $\adjfunc{\tau}^{-1}$.}
%    \label{table:adjustment}
%  \end{table}

  % \begin{table}[t]
  %   \centering
  %   \begin{tabular}{lccccc}
  %     \toprule
  %     \textsc{Adjustment} & Linear & Diff \\
  %     \midrule
  %     Linear & $\tau \lambda$ & $\tau + \lambda$ \\
  %     $\adjfunc{\tau}(\lambda)$ & $\tau \lambda$ & $\tau + \lambda$ \\
  %     $\adjfunc{\tau}^{-1}(\lambda)$ & $\tau^{-1} \lambda$ & $\lambda - \tau$ \\
  %     \bottomrule
  %   \end{tabular}
  %   \vskip 0.08in
  %   \caption{Adjustment Functions $\adjfunc{t}$ and their inverses $\adjfunc{\tau}^{-1}$.}
  %   \label{table:adjustment}
  % \end{table}

  \begin{table}[t]
    \centering
    \begin{tabular}{lccc}
      \toprule
      \textsc{Name} & $\adjfunc{\tau}(\lambda)$ & $\adjfunc{\tau}^{-1}(\lambda)$ & $\varphi$ \\
      \midrule
      Linear & $\tau \lambda$ & $\tau^{-1} \lambda$ & 1 \\
      Difference & $\tau + \lambda$ & $\lambda - \tau$ & 0 \\
      \bottomrule
    \end{tabular}
    \vskip 0.08in
    \caption{Adjustment Functions $\adjfunc{t}$, their inverses $\adjfunc{\tau}^{-1}$ and $\varphi$ values used in our experiments.}
    \label{table:adjustment}
  \end{table}

\subsection{Details of the experimental setup}
  We use the Mixture Density Network~\citep{bishop1994mixture} implementation from CDE~\citep{rothfuss2019conditional} Python package\footnote{\url{https://github.com/freelunchtheorem/Conditional_Density_Estimation}} as a base model for \texttt{CP}, \texttt{PCP} and \algo. The underlying neural network contains two hidden layers of 100 neurons each and was trained for 1000 epochs for each split of the data. Number of components of the Gaussian Mixture was set to 10 for all datasets.

  For the \texttt{CQR}~\citep{romano2019conformalized} and \texttt{CHR}~\citep{sesia2021conformal} we use the original authors' implementation\footnote{\url{https://github.com/msesia/chr}}. The underlying neural network that outputs conditional quantiles consists of two hidden layers with 64 neurons each. Training was performed for 200 epochs for batch size 250. 

  We replicate the experiments for 50 random splits of all nine datasets. To lower noise in calculated performance metrics we reuse trained networks and samples across different top-level algorithms for each replication.

\subsection{Worst-slab coverage}
\label{appendix:worst_slab}
  Here we present some additional experiments related to conditional coverage achieved by different methods. We have used Worst Slab Coverage metric, which is sensitive to the set of labs considered during the search. Following~\citet{cauchois2020knowing, romano2020classification}, recall that a slab is defined as
  \begin{equation*}
    S_{v, a, b} = \left\{ x \in \mathbb{R}^p: a < v^Tx < b \right\},
  \end{equation*}
  where $v \in \mathbb{R}^p$ and $a,b \in\mathbb{R}$, such that $a<b$. Now, given the prediction set $\mathcal{C}(x)$ and $\delta \in [0, 1]$, the \textit{worst-slab coverage} is defined as:
  \begin{equation*}
    \mathrm{WSC}(\mathcal{C}, \delta) =  \inf_{v \in \mathbb{R}^p, a<b\in\mathbb{R}} \prob\left( Y \in \mathcal{C}(X) | X \in S_{v, a, b} \right) \: s.t. \: \prob(X \in S_{v, a, b}) \ge 1 - \delta. 
  \end{equation*}

 In our experiments we follow~\cite{romano2020classification} in our implementation of this metric. Namely, we use $25\%$ of the data to find the worst slab and the use the remaining $75\%$ to calculate the final value on this slab. We use $5000$ randomly sampled directions, that are the same for each algorithm and change for each replication.

 % In \Cref{sec:expriments}, we presented the results obtained for $(1-\delta)=0.1$. In this case, the considered slabs must contain at least $10\%$ of the data. In Figure~\ref{fig:real_wsc04} we report results obtained for $(1-\delta)=0.4$. We can see that performance improves a lot compared to $\delta=0.1$, and most results become indistinguishable.

 % \begin{figure}[t!]
 %   \centering
 %   \includegraphics[width=\textwidth]{real/k50/cond_cov_04_5k_50_2000_2000_50_delta_0-4_zoom.pdf}
 %   \caption{Worst-slab coverage on real data, $(1-\delta) = 0.4$.}
 %   \label{fig:real_wsc04}
 % \end{figure}

\subsection{Extended results of real data experiments}

  Table~\ref{table:combined_real} we summarize all metrics from our real-world data experiments. For conditional coverage we report worst-slab coverage with $(1-\delta)=0.1$. On six out of nine datasets \algo\ method achieves the best result in conditional coverage. In terms of interval width PCP method produces the narrowest intervals.As we can see, it happens at the expense of conditional coverage: PCP often achieves significantly lower values.

  We also present a more detailed view of set size differences between the methods. In the main part we reported average rank of each method in Figure~\ref{fig:md_2d}. We ranked the algorithms by their projected area at each test point and averaged the ranks. Here we show raw areas of the projections onto each pairs of axes for $\texttt{sgemm\_small}$ dataset in table~\ref{table:areas_sgemm}. All targets were standardised to zero mean and unit standard deviation so that different projections will be in the same scale. We see that $\texttt{PCP}$ produces smaller set sizes like in one-dimensional case. Quantile-regression based methods have the largest sets, even larger than the fixed-sized sets of $\texttt{CP}$. Our approach demonstrates only modest increase in prediction set size compared to $\texttt{PCP}$ while achieving sharper conditional coverage.

\begin{table}[h]
\centering
\caption{Summary results of experiments on real data. ``M. Cov.'' stands for marginal coverage, ``C. Cov.'' is the worst-slab coverage (here $(1-\delta)=0.1$), and $w_{sd}$ is average total length of the prediction sets, scaled by standard deviation of $Y$. Nominal coverage level is set to $(1 - \alpha) = 0.9$. For $\gammayx$, $\texttt{PCP}$, \pcpalgo\ we use the same underlying mixture density network model with $50$ samples. $\texttt{CHR}$ and $\texttt{CQR}(2)$ also share the same base neural network model. We average results of $50$ random data splits. For each dataset, we highlighted the algorithm achieving conditional coverage closest to the nominal level.}
\label{table:combined_real}
\begin{tabular}{l|l|r|r|r|r|r|r|r|r}
\toprule
& & & & & \multicolumn{2}{c|}{$\texttt{CP}^2$} & & \\
Dataset & Metric & $\texttt{CP}$ & $\texttt{PCP}$ & $\Pi_{Y|X}$ & $\texttt{PCP}$-$\texttt{L}$ & $\texttt{PCP}$-$\texttt{D}$ & $\texttt{CHR}$ & $\texttt{CQR}$ & $\texttt{CQR2}$ \\
\midrule
\multirow[|c|]{3}{*}{bike} & M. Cov. & 0.90 & 0.90 & 0.93 & 0.90 & 0.90 & 0.90 & \cellcolor{lightgray} 0.90 & 0.90 \\
 & C. Cov. & 0.79 & 0.85 & 0.92 & 0.89 & 0.89 & 0.88 & \cellcolor{lightgray} 0.90 & 0.87 \\
 & $w_{sd}$ & \bfseries 0.71 & 0.71 & 0.83 & 0.79 & 0.80 & 1.94 & \cellcolor{lightgray} 2.25 & 2.31 \\
\midrule
\multirow[|c|]{3}{*}{bio} & M. Cov. & 0.90 & 0.90 & 0.91 & 0.90 & 0.90 & \cellcolor{lightgray} 0.90 & 0.90 & 0.90 \\
 & C. Cov. & 0.88 & 0.89 & 0.91 & 0.90 & 0.90 & \cellcolor{lightgray} 0.90 & 0.89 & 0.89 \\
 & $w_{sd}$ & 2.34 & \bfseries 1.89 & 1.95 & 1.97 & 1.95 & \cellcolor{lightgray} 1.92 & 2.13 & 2.10 \\
\midrule
\multirow[|c|]{3}{*}{blog} & M. Cov. & 0.90 & 0.90 & 0.91 & 0.90 & \cellcolor{lightgray} 0.90 & 0.90 & 0.90 & 0.90 \\
 & C. Cov. & 0.60 & 0.74 & 0.91 & 0.89 & \cellcolor{lightgray} 0.90 & 0.87 & 0.87 & 0.86 \\
 & $w_{sd}$ & 0.60 & \bfseries 0.30 & 0.72 & 0.72 & \cellcolor{lightgray} 0.71 & 0.31 & 0.44 & 0.39 \\
\midrule
\multirow[|c|]{3}{*}{fb1} & M. Cov. & 0.90 & 0.90 & 0.93 & 0.90 & 0.90 & 0.90 & \cellcolor{lightgray} 0.90 & 0.90 \\
 & C. Cov. & 0.49 & 0.64 & 0.92 & 0.88 & 0.89 & 0.87 & \cellcolor{lightgray} 0.90 & 0.87 \\
 & $w_{sd}$ & 0.47 & 0.28 & 0.58 & 0.56 & 0.59 & \bfseries 0.26 & \cellcolor{lightgray} 0.37 & 0.33 \\
\midrule
\multirow[|c|]{3}{*}{fb2} & M. Cov. & 0.90 & 0.90 & \cellcolor{lightgray} 0.93 & 0.90 & 0.90 & 0.90 & 0.90 & 0.90 \\
 & C. Cov. & 0.50 & 0.61 & \cellcolor{lightgray} 0.91 & 0.88 & 0.88 & 0.88 & 0.89 & 0.89 \\
 & $w_{sd}$ & 0.53 & \bfseries 0.32 & \cellcolor{lightgray} 0.65 & 0.62 & 0.65 & 0.33 & 0.43 & 0.37 \\
\midrule
\multirow[|c|]{3}{*}{meps19} & M. Cov. & 0.90 & 0.90 & 0.89 & \cellcolor{lightgray} 0.90 & 0.90 & 0.90 & 0.89 & 0.90 \\
 & C. Cov. & 0.54 & 0.78 & 0.89 & \cellcolor{lightgray} 0.90 & 0.89 & 0.90 & 0.88 & 0.89 \\
 & $w_{sd}$ & 1.05 & \bfseries 0.73 & 1.02 & \cellcolor{lightgray} 1.19 & 1.07 & 0.76 & 1.14 & 1.19 \\
\midrule
\multirow[|c|]{3}{*}{meps20} & M. Cov. & 0.90 & 0.90 & 0.89 & 0.90 & \cellcolor{lightgray} 0.90 & 0.90 & 0.90 & 0.90 \\
 & C. Cov. & 0.58 & 0.80 & 0.89 & 0.90 & \cellcolor{lightgray} 0.90 & 0.91 & 0.88 & 0.89 \\
 & $w_{sd}$ & 1.06 & \bfseries 0.75 & 0.98 & 1.15 & \cellcolor{lightgray} 1.04 & 0.77 & 1.09 & 1.17 \\
\midrule
\multirow[|c|]{3}{*}{meps21} & M. Cov. & 0.90 & 0.90 & 0.89 & 0.90 & 0.90 & \cellcolor{lightgray} 0.90 & 0.90 & 0.90 \\
 & C. Cov. & 0.54 & 0.81 & 0.89 & 0.89 & 0.89 & \cellcolor{lightgray} 0.90 & 0.89 & 0.88 \\
 & $w_{sd}$ & 1.04 & \bfseries 0.72 & 0.99 & 1.16 & 1.04 & \cellcolor{lightgray} 0.79 & 1.13 & 1.21 \\
\midrule
\multirow[|c|]{3}{*}{temp} & M. Cov. & 0.90 & 0.90 & 0.82 & 0.90 & \cellcolor{lightgray} 0.90 & 0.90 & 0.90 & 0.90 \\
 & C. Cov. & 0.87 & 0.89 & 0.81 & 0.88 & \cellcolor{lightgray} 0.89 & 0.86 & 0.85 & 0.86 \\
 & $w_{sd}$ & 0.87 & 0.92 & \bfseries 0.78 & 0.96 & \cellcolor{lightgray} 0.93 & 1.31 & 1.48 & 1.30 \\
\bottomrule
\end{tabular}
\end{table}

\begin{table}[h]
\centering
\caption{Prediction set size comparison for $\texttt{sgemm\_small}$ dataset. Rows correspond to different pairs of targets (dataset has 4 targets). For each method the reported value is the mean area of the 2D projection of the prediction set to the corresponding axes pair.}
\label{table:areas_sgemm}
\begin{tabular}{l|cccccccc}
\toprule
& & & & \multicolumn{2}{|c|}{$\texttt{CP}^2$} & & & \\
Axes & $\texttt{CP}$ & $\texttt{PCP}$ & $\Pi_{Y|X}$ & $\texttt{PCP}$-$\texttt{L}$ & $\texttt{PCP}$-$\texttt{D}$ & $\texttt{CHR}$ & $\texttt{CQR}$ & $\texttt{CQR2}$ \\
\midrule
(0, 1) & 2.137 & 0.435 & 0.517 & 0.576 & 0.560 & 2.290 & 2.550 & 2.436 \\
(0, 2) & 2.145 & 0.435 & 0.518 & 0.577 & 0.561 & 2.267 & 2.506 & 2.358 \\
(0, 3) & 2.145 & 0.436 & 0.519 & 0.578 & 0.561 & 2.086 & 2.366 & 2.172 \\
(1, 2) & 2.146 & 0.435 & 0.517 & 0.576 & 0.560 & 2.388 & 2.622 & 2.546 \\
(1, 3) & 2.146 & 0.435 & 0.517 & 0.576 & 0.560 & 2.166 & 2.461 & 2.314 \\
(2, 3) & 2.154 & 0.436 & 0.519 & 0.578 & 0.562 & 2.153 & 2.430 & 2.255 \\
\end{tabular}

\end{table}

\subsection{Other perspective on conditional coverage}
  The worst-slab coverage metric used in the previous section is not always helpful: (1) it provides a single number for each method, and (2) the selected slab is different for each algorithm. In practice we might be interested in how sharp the coverage is along the portion of the input space spanned by the test data. To explore this, we used two approaches: dimensionality reduction and clustering. Results for clustering with HDBSCAN are presented in the main part in Figure~\ref{fig:ccov_f1_hdb}, here turn to dimensionality reduction.

  First we apply UMAP algorithm to project data to two dimensions and then construct a heatmap plot to show coverage in each bin of the histogram. Results for \texttt{meps\_19} dataset are presented in Figure~\ref{fig:ccov_meps19hist}. Nominal coverage is set to $(1-\alpha)=0.9$ and corresponds to gray part of the color scale. We can see that our method and baseline $\gammayx$ perform better than \texttt{CP} and \texttt{PCP} across the space.

  \begin{figure}[t!]
    \centering
    \includegraphics[width=\textwidth]{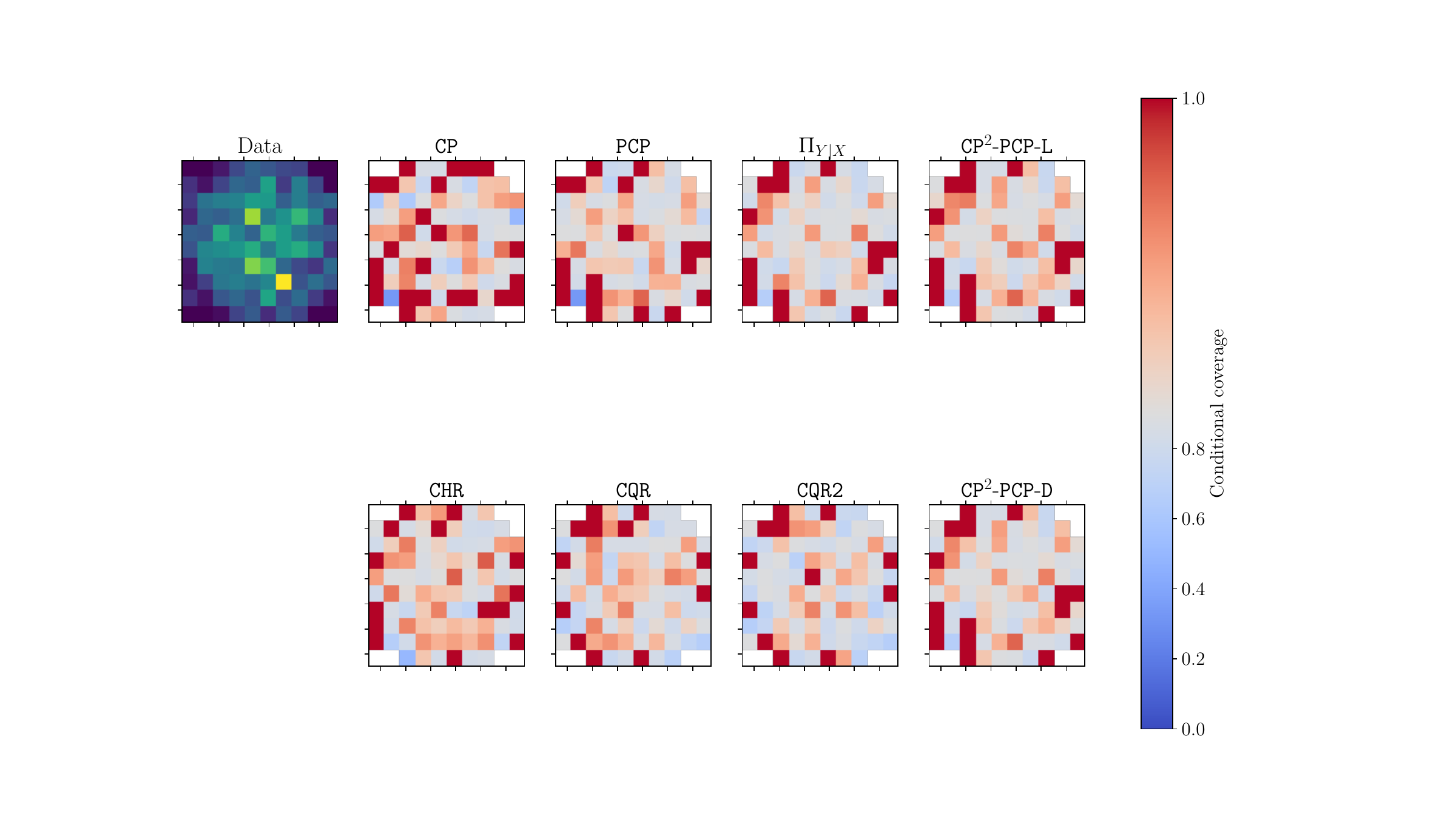}
    \caption{Conditional coverage after dimensionality reduction, \texttt{meps\_21} dataset. Data projected to two dimensions using UMAP algorithm with Canberra metric, with the \texttt{n\_neighbors} hyperparameter set to 2. Nominal coverage is set to $(1-\alpha)=0.1$, it corresponds to gray on the color scale.}
    \label{fig:ccov_meps19hist}
  \end{figure}

\end{document}